%% file: ms.tex
\pdfoutput=1
\documentclass{article}

\usepackage{microtype}
\usepackage{graphicx}
\usepackage{booktabs} 
\usepackage{amsthm}
\usepackage{thmtools}
\usepackage{thm-restate}
\usepackage{caption,subcaption}

\usepackage{todonotes}

\theoremstyle{plain}

\usepackage[algo2e,lined]{algorithm2e}

\SetCommentSty{mycommfont}

\usepackage{hyperref}

\usepackage{lipsum}

\newcommand{\gup}{f}
\newcommand{\glo}{g}

\newcommand{\priori}[1]{p(\theta_{#1})}

\newcommand{\likelihood}{p(h_T|\theta, \pi_\phi)}
\newcommand{\likelihoodi}[1]{p(h_T|\theta_{#1}, \pi_\phi)}

\usepackage[accepted]{icml2021}

\input{math_commands}

\makeatletter
\def\tagform@#1{\maketag@@@{(\ignorespaces#1\unskip\@@italiccorr)}}
\renewcommand{\eqref}[1]{\textup{{\normalfont(\ref{#1}}\normalfont)}}
\makeatother

\icmltitlerunning{Deep Adaptive Design: Amortizing Sequential Bayesian Experimental Design}

\begin{document}

\twocolumn[
\icmltitle{Deep Adaptive Design: Amortizing Sequential Bayesian Experimental Design}



\icmlsetsymbol{equal}{*}

\begin{icmlauthorlist}
\icmlauthor{Adam Foster}{equal,ox}
\icmlauthor{Desi R. Ivanova}{equal,ox}
\icmlauthor{Ilyas Malik}{student}
\icmlauthor{Tom Rainforth}{ox}
\end{icmlauthorlist}

\icmlaffiliation{ox}{Department of Statistics, University of Oxford, UK}
\icmlaffiliation{student}{Work undertaken whilst at the University of Oxford}

\icmlcorrespondingauthor{Adam Foster}{adam.foster@stats.ox.ac.uk}

\icmlkeywords{Machine Learning, ICML, Bayesian Experimental Design, Optimal Experimental Design,
	Amortized Inference}

\vskip 0.3in
]



\printAffiliationsAndNotice{\icmlEqualContribution} 

\input{abstract}

\input{introduction}

\input{background}

\input{reformulation}

\input{method}
\input{relwork}

\input{experiments}

\input{discussion}

\input{acknowledgements}

\bibliographystyle{icml2021}
\bibliography{ms}

\clearpage

\appendix
\input{appendix}

\end{document}

%% file: math_commands.tex
\usepackage{amsmath,amsfonts,amssymb,bm}

\newcommand{\iid}{\overset{\scriptstyle{\text{i.i.d.}}}{\sim}}

\def\eqref#1{equation~\ref{#1}}

\def\1{\bm{1}}

\newcommand{\pypx}[2]{\frac{\partial #1}{\partial #2}}
\newcommand{\ddx}[1]{\frac{d}{d #1}}
\newcommand{\dydx}[2]{\frac{d #1}{d #2}}

\DeclareMathAlphabet{\mathsfit}{\encodingdefault}{\sfdefault}{m}{sl}
\SetMathAlphabet{\mathsfit}{bold}{\encodingdefault}{\sfdefault}{bx}{n}

\newcommand{\E}{\mathbb{E}}

\newcommand{\R}{\mathbb{R}}

\DeclareMathOperator*{\argmax}{arg\,max}

%% file: abstract.tex

\begin{abstract}
We introduce \emph{Deep Adaptive Design} (DAD), a method for amortizing the cost of adaptive Bayesian experimental design that allows experiments to be run in real-time. 
Traditional sequential Bayesian optimal experimental design approaches require substantial computation at \emph{each} stage of the experiment.
This makes them unsuitable for most real-world applications, where decisions must typically be made quickly.
DAD addresses this restriction by learning an amortized \emph{design network} upfront and then using this to rapidly run (multiple) adaptive experiments at deployment time.
This network represents a design \emph{policy} which takes as input the data from previous steps, and outputs the next design using a single forward pass; these design decisions can be made in milliseconds during the live experiment.
To train the network, we introduce contrastive information bounds that are suitable objectives for the sequential setting, and propose a customized network architecture that exploits key symmetries.
We demonstrate that DAD successfully amortizes the process of experimental design, outperforming alternative strategies on a number of problems.
\end{abstract}

%% file: introduction.tex

\section{Introduction}
\label{sec:intro}

A key challenge across disciplines as diverse as psychology \citep{myung2013}, bioinformatics \citep{vanlier2012}, pharmacology \citep{lyu2019ultra} and physics \citep{dushenko2020sequential} is to design experiments so that the outcomes will be as informative as possible about the underlying process. Bayesian optimal experimental design (BOED) is a powerful mathematical framework for tackling this problem~\citep{lindley1956,chaloner1995}. 

In the BOED framework, outcomes $y$ are modeled in a Bayesian manner \citep{gelman2013bayesian,kruschke2014doing} using a likelihood $p(y|\theta,\xi)$ and a prior $p(\theta)$, where $\xi$ is our controllable design and $\theta$ is the set of parameters we wish to learn about. We then optimize $\xi$ to maximize the \emph{expected information gained} about $\theta$ (equivalently the mutual information between $y$ and $\theta$):
\begin{equation}
\label{eq:mi}
I(\xi) := \E_{p(\theta)p(y|\theta,\xi)}\left[ \log p(y|\theta,\xi) -\log p(y|\xi)\right].
\end{equation}
The true power of BOED is realized when it is used to design a sequence of experiments $\xi_1,...,\xi_T$, wherein it allows us to construct \emph{adaptive} strategies which utilize information gathered from past data 
to tailor each successive design $\xi_t$ during the progress of the experiment. The conventional, iterative, approach for selecting each $\xi_t$ is to fit the posterior $p(\theta|\xi_{1:t-1},y_{1:t-1})$ representing the updated beliefs about $\theta$ after $t-1$ iterations have been conducted, and then substitute this for the prior in \eqref{eq:mi}~\citep{ryan2016review,rainforth2017thesis,kleinegesse2020sequential}. The design $\xi_t$ is then chosen as the one which maximizes the resulting objective.

Unfortunately, this approach necessitates significant computational time to be expended \emph{between each step of the experiment} in order to update the posterior and compute the next optimal design. 
In particular, 
$I(\xi)$ is doubly intractable \citep{rainforth2018nesting,zheng2018robust} and its optimization constitutes a significant computational bottleneck.  This can be prohibitive to the practical application of sequential BOED as 
design decisions usually need to be made quickly for  the approach to be useful \citep{evans2005value}.

To give a concrete example, consider running an adaptive survey to understand political opinions~\citep{pasek2010optimizing}.  A question $\xi_t$ is put to a participant who gives their answer $y_t$ and this data is used to update an underlying model with latent variables $\theta$. Here sequential BOED is of immense value because previous answers can be used to guide future questions, ensuring that they are pertinent to the particular participant. However, it is not acceptable to have lengthy delays between questions to compute the next design, precluding existing approaches from being used.

To alleviate this problem, we propose \emph{amortizing} the cost of sequential experimental design, performing upfront training before the start of the experiment to allow very fast design decisions at deployment, when time is at a premium. This amortization is particularly useful in the common scenario where the same adaptive experimental framework will be deployed numerous times (e.g.~having multiple participants in a survey).  
Here amortization not only removes the computational burden from the live experiment, it also allows for sharing computation across multiple experiments, analogous to inference amortization that allows one to deal with multiple datasets~\citep{stuhlmuller2013learning}.

Our approach, called \textbf{Deep Adaptive Design (DAD)}, constructs a single \emph{design network} which takes as input the designs and observations from previous stages, and outputs the design to use for the next experiment. The network is learned by simulating hypothetical experimental trajectories and then using these to train the network to make near-optimal design decisions automatically. That is, it learns a \emph{design policy} which makes decisions as a function of the past data, and we optimize the parameters of this policy rather than an individual design. 
Once learned, the network eliminates the computational bottleneck at each iteration of the experiment, enabling it to be run both adaptively and quickly; it can also be used repeatedly for different instantiations of the experiment (e.g.~different human participants).

To allow for efficient, effective, and simple training, we show how DAD networks can be learned without any direct posterior or marginal likelihood estimation.  
This is achieved by reformulating the sequential BOED problem from its conventional iterative form, to a single holistic objective based on the \emph{overall} expected information gained from the entire experiment when using a policy to make each design decision deterministically given previous design outcome pairs.
We then derive contrastive bounds on this objective that allow for end-to-end training of the policy parameters with stochastic gradient ascent, thereby sidestepping both the need for inference and the double intractability of the EIG objective.
This approach has the further substantial benefit of allowing non-myopic adaptive strategies to be learned, that is strategies which take account of their own future decisions, unlike conventional approaches.

We further demonstrate a key permutation symmetry property of the optimal design policy, and use this to propose a customized architecture for the experimental design network. This is critical to allowing  effective amortization across time steps. The overall result of the theoretical formulation, novel contrastive bounds, and neural architecture is a methodology which enables us to bring the power of deep learning to bear on adaptive experimental design.

We apply DAD to a range of problems relevant to applications such as epidemiology, physics and psychology. We find that DAD is able to accurately amortize experiments, opening the door to running adaptive BOED in real time.

%% file: background.tex

\vspace{-5pt}
\section{Background}
\label{sec:background}
\vspace{-3pt}

Because experimentation is a potentially costly endeavour, it is essential to design experiments in manner that maximizes the amount of information garnered.
The BOED framework, pioneered by \citet{lindley1956}, provides a powerful means of doing this in a principled manner.
Its key idea is to optimize the experimental design $\xi$ to maximize the expected amount of \emph{information} that will be gained about the latent variables of interest, $\theta$, upon observing the experiment outcome $y$. 

To implement this approach, we begin with the standard Bayesian modelling set-up consisting of an explicit likelihood model $p(y|\theta,\xi)$ for the experiment, and a prior $p(\theta)$ representing our initial beliefs about the unknown latent. After running a hypothetical experiment with design $\xi$ and observing $y$, our updated beliefs are the posterior $p(\theta|\xi,y)$. 
The amount of information that has been gained about $\theta$ can be mathematically described by the reduction in entropy from the prior to the posterior
\begin{equation}
	\text{IG}(\xi,y) = \text{H}\left[p(\theta)\right] - \text{H}\left[p(\theta|\xi,y) \right].
\end{equation}
The \emph{expected information gain} (EIG) is formed by taking the expectation over possible outcomes $y$, using the model itself to simulate these.  Namely we take an expectation with respect to $y\sim p(y|\xi)=\E_{p(\theta)}[p(y|\theta,\xi)]$, yielding
\begin{align*}
	I(\xi) :=&\, \E_{p(y|\xi)}\left[\text{IG}(\xi,y) \right] \\
	=&\,\E_{p(\theta)p(y|\theta,\xi)}\left[ \log p(\theta|\xi,y)-\log p(\theta)  \right]\\
	=&\,\E_{p(\theta)p(y|\theta,\xi)}\left[ \log p(y|\theta,\xi) - \log p(y|\xi) \right]
\end{align*}
which is the mutual information between $y$ and $\theta$ under design $\xi$. The optimal design is defined as $\xi^* = \argmax_{\xi\in\Xi} I(\xi)$, where $\Xi$ is the space of feasible designs.

It is common in BOED settings to be able to run multiple experiment iterations with designs $\xi_1,...,\xi_T$, observing respective outcomes $y_1,...,y_T$. 
One simple strategy for this case is \emph{static} design, also called fixed or batch design, which selects all $\xi_1,...,\xi_T$ before making any observation. 
The designs are optimized to maximize the EIG, with $y_{1:T}$ in place of $y$ and $\xi_{1:T}$ in place of $\xi$, effectively treating the whole sequence of experiments as one experiment with enlarged observation and design spaces.

\subsection{Conventional adaptive BOED}
\label{sec:conventional}
This static design approach is generally sub-optimal as it ignores the fact that information from previous iterations can substantially aid in the design decisions at future iterations.
The power of the BOED framework can thus be significantly increased by using an \emph{adaptive} design strategy that chooses each $\xi_t$ dependent upon $\xi_{1:t-1},y_{1:t-1}$. This enables us to use what has already been learned in previous experiments to design the next one optimally, resulting in a virtuous cycle of refining beliefs and using our updated beliefs to design good experiments for future iterations.

The conventional approach to computing designs adaptively is to fit the posterior distribution $p(\theta|\xi_{1:t-1},y_{1:t-1})$ at each step, 
and then optimize the EIG objective that uses this posterior in place of the prior~\citep{ryan2016review}
\begin{align}
	\label{eq:i_t}
	I(\xi_t) = \E_{p(\theta|\xi_{1:t-1},y_{1:t-1})p(y_t|\theta,\xi_t)}\left[ \log \frac{p(y_t|\theta,\xi_t)}{p(y_t|\xi_t)} \right]
\end{align}
where $p(y_t|\xi_t)=\E_{p(\theta|\xi_{1:t-1},y_{1:t-1})}[p(y_t|\theta,\xi_t)]$. 

Despite the great potential of the adaptive BOED framework, this conventional approach is very computationally expensive. 
At each stage $t$ of the experiment we must compute the posterior $p(\theta|\xi_{1:t-1},y_{1:t-1})$, which is costly and cannot be done in advance as it depends on $y_{1:t-1}$.  
Furthermore, the posterior is then used to obtain $\xi_t$ by maximizing the objective in~\eqref{eq:i_t}, which is computationally even more demanding as it involves the optimization of a doubly intractable quantity \citep{rainforth2018nesting,foster2019variational}.  Both of these steps must be done during the experiment, meaning it is infeasible to run adaptive BOED in real time experiment settings unless the model is unusually simple.

\subsection{Contrastive information bounds}
\label{sec:cib}
In \citet{foster2020unified}, the authors noted that if $\xi\in\Xi$ is continuous, approximate optimization of the EIG at each stage of the experiment can be achieved in a single \emph{unified} stochastic gradient procedure that both estimates and optimizes the EIG simultaneously. 
A key component of this approach is the derivation of several contrastive lower bounds on the EIG, inspired by work in representation learning \citep{oord2018representation,poole2018variational}.
One such bound is the Prior Contrastive Estimation (PCE) bound, given by
\begin{equation}
	I(\xi) \ge \E\left[\log \frac{p(y|\theta_0,\xi)}{\frac{1}{L+1}\sum_{\ell=0}^L p(y|\theta_\ell,\xi)} \right]
\end{equation}
where $\theta_0\sim p(\theta)$ is the sample used to generate $y\sim p(y|\theta,\xi)$ and $\theta_{1:L}$ are $L$ contrastive samples drawn independently from $p(\theta)$; as $L\to\infty$ the bound becomes tight.
The PCE bound can be maximized by stochastic gradient ascent (SGA) \citep{robbins1951stochastic} to approximate the optimal design $\xi$.
As discussed previously, in a sequential setting this stochastic gradient optimization is repeated $T$ times, with $p(\theta)$ replaced by $p(\theta|\xi_{1,t-1},y_{1:t-1})$ at step $t$.

%% file: reformulation.tex

\section{Rethinking Sequential BOED}
\label{sec:reformulation}

To enable adaptive BOED to be deployed in settings where design decisions must be taken quickly,  we first need to rethink the traditional iterative approach to produce a formulation which considers the entire design process holistically.
To this end, we introduce the concept of a \emph{design function}, or \emph{policy}, $\pi$ that maps from the set of all previous design--observation pairs to the next chosen
design.

Let $h_t$ denote the experimental \emph{history} $(\xi_1,y_1),...,(\xi_t,y_t)$. We can simulate histories for a given policy $\pi$, by sampling a $\theta\sim p(\theta)$, then, for each $t=1,...,T$, fixing $\xi_t = \pi(h_{t-1})$ (where $h_0=\varnothing$) and sampling $y_t \sim p(y|\theta,\xi_t)$.  The density of this generative process can be written as
\begin{equation}
\label{eq:jlikelihood}
p(\theta)p(h_T|\theta,\pi) = p(\theta)\prod\nolimits_{t=1}^T p(y_t|\theta,\xi_t).
\end{equation}
The standard sequential BOED approach  described in \S~\ref{sec:conventional} now corresponds to a costly implicit policy $\pi_s$, that performs posterior estimation followed by EIG optimization to choose each design.  By contrast, in DAD, we will learn a deterministic $\pi$ that chooses designs directly.

Another way to think about $\pi_s$ is that it is the policy which piecewise optimizes the following objective for $\xi_t | h_{t-1}$
\begin{equation}
\label{eq:reward}
I_{h_{t-1}}(\xi_t) := \E_{p(\theta|h_{t-1})p(y_t|\theta,\xi_t)}\left[ \log \frac{p(y_t|\theta,\xi_t)}{p(y_t|h_{t-1},\xi_t)} \right]
\end{equation}
 where $p(y_t|h_{t-1},\xi_t)=\E_{p(\theta|h_{t-1})}[p(y_t|\theta,\xi_t)]$. It is thus the optimal \emph{myopic} policy---that is a policy which fails to reason about its own future actions---for an objective given by the sum of EIGs from each experiment iteration. Note that this is not the optimal \emph{overall} policy as it fails to account for future decision making: some designs may allow better future design decisions than others
 than others \citep{gonzalez2016glasses,jiang2020binoculars}.\footnote{To give an intuitive example, consider
 the problem of placing two breakpoints on the line $[0,1]$ to produce the most evenly sized segments.
 The optimal myopic policy places its first design at
 $1/2$ and its second at either $1/4$ or $3/4$.  This is suboptimal since the best strategy is to place the two breakpoints
 at $1/3$ and $2/3$.}

Trying to learn an efficient policy that directly mimics $\pi_s$ would be very computationally challenging because of the difficulties of dealing with both inference and EIG estimation at each iteration of the training. Indeed, the natural way to do this involves running a full, very expensive, simulated sequential BOED process to generate each training example.

We instead propose a novel strategy that reformulates the sequential decision problem in a way that  completely eliminates the need for calculating either posterior distributions or intermediate EIGs, while also allowing for
non-myopic policies to be learned.  
This is done by exploiting an important property of the EIG: the total EIG of a sequential experiment is the sum of the (conditional) EIGs for each experiment iteration. This is formalized in the following result, which provides a single expression for the expected information gained from the entire sequence of $T$ experiments.
\vspace{-9pt}
\begin{restatable}{theorem}{terminal}
	\label{proposition:terminal}
	The total expected information gain for policy $\pi$ over a sequence of $T$ experiments is
	\begin{align}
	\mathcal{I}_T(\pi) &:= \E_{p(\theta)p(h_T|\theta,\pi)} \left[\sum\nolimits_{t=1}^{T} I_{h_{t-1}}(\xi_t)\right] \\
	=&\,\E_{p(\theta)p(h_T|\theta,\pi)}\left[ \log p(h_T|\theta,\pi) - \log p(h_T|\pi)\right]
		\label{eq:objective}
	\end{align}
	where $p(h_T|\pi) = \E_{p(\theta)}[p(h_T|\theta,\pi)]$.
\end{restatable}
\vspace{-2pt}
The proof is given in Appendix~\ref{sec:proofs}. Intuitively, 
$\mathcal{I}_T(\pi)$ is the expected reduction in entropy from the prior $p(\theta)$ to the \emph{final} posterior $p(\theta|h_T)$, without considering the intermediate posteriors at all.
Note here a critical change from previous BOED formulations: $\mathcal{I}_T(\pi)$ is a function of the policy, not the designs themselves, with the latter now being random variables (due to their dependence on previous outcomes) that we take an expectation over.
This is actually a strict generalization of conventional BOED frameworks: static design corresponds to policy that consists of $T$ fixed designs with no adaptivity, for which~\eqref{eq:objective} coincides with $I(\xi_{1:T})$, while conventional adaptive BOED approximates $\pi_s$.

By reformulating our objective in terms of a policy, we have constructed a single end-to-end objective for adaptive, non-myopic design and which requires negligible computation at deployment time: once $\pi$ is learned, it can just be directly evaluated during the experiment itself.

%% file: method.tex

\vspace{-2pt}
\section{Deep Adaptive Design}
\label{sec:dad}
Theorem~\ref{proposition:terminal} showed that the optimal design function $\pi^* = \argmax_\pi \mathcal{I}_T(\pi)$ is the one which maximizes the mutual information between the unknown latent $\theta$ and the full rollout of histories produced using that policy, $h_T$.
DAD looks to approximate $\pi^*$ explicitly using a neural network, which we  now refer to as the \emph{design network} $\pi_\phi$, with trainable parameters $\phi$. 
This policy-based approach marks a major break from existing methods, which do not represent design decisions explicitly as a function, but instead optimize designs on the fly during the experiment.

DAD amortizes the cost of experimental design---by training the network parameters $\phi$, the design network is taught to make correct design decisions across a wide range of possible experimental outcomes. This removes the cost of adaptation for the live experiment itself: during deployment the design network will select the next design nearly instantaneously with a single forward pass of the network.
Further, it offers a simplification and streamlining of the sequential BOED process: it only requires the upfront end-to-end training of a single neural network and thus negates the need to set up complex \emph{automated} inference and optimization schemes that would otherwise have to run in the background during a live experiment. 
A high-level summary of the DAD approach is given in Algorithm~\ref{algo:DAD}.

Two key technical challenges still stand in the way of realizing the potential of adaptive BOED in real time.
First, whilst the unified objective $\mathcal{I}_T(\pi)$ does not require the computation of intermediate posterior distributions, it remains an intractable objective due to the presence of $p(h_T|\pi)$.
To deal with this, we derive a family of lower bounds that are appropriate for the policy-based setting and use them
to construct stochastic gradient training schemes for $\phi$.
Second,  to ensure that this network can efficiently learn a mapping from histories to designs, we require an effective architecture. 
As we show later, the optimal policy is invariant to the order of the history, and we use this key symmetry to architect an effective design network.

{\setlength{\textfloatsep}{0pt}
\begin{algorithm}[t]
\SetAlgoLined
\SetKwInput{Input}{Input}
\SetKwInput{Output}{Output}
\Input{Prior $p(\theta)$, likelihood $p(y|\theta,\xi)$, number of steps $T$}
\Output{Design network $\pi_\phi$}
\While{\textnormal{training compute budget not exceeded}}{
    Sample $\theta_0 \sim p(\theta)$ and set $h_0=\varnothing$\;
    
    \For {$t=1,...,T$}{
    Compute $\xi_t = \pi_\phi(h_{t-1})$\;
    
    Sample $y_t \sim p(y|\theta_0,\xi_t)$\;
    
    Set $h_t=\{(\xi_1,y_1),...,(\xi_t,y_t)\}$ 
    }
    Compute estimate for ${d\mathcal{L}_T}/{d\phi}$ as per \S~\ref{sec:gradientestimation}
    
    Update $\phi$ using stochastic gradient ascent scheme
}
At deployment, $\pi_\phi$ is fixed, we take $\xi_t = \pi_\phi(h_{t-1})$, and each $y_t$ is obtained by running an experiment with $\xi_t$.
\caption{Deep Adaptive Design (DAD)}
\label{algo:DAD}
\end{algorithm}
}

\subsection{Contrastive bounds for sequential experiments}

Our high-level aim is to train $\pi_\phi$ to maximize the mutual information $\mathcal{I}_T(\pi_\phi)$.
In contrast to most machine learning tasks, this objective is \emph{doubly} intractable and cannot be directly evaluated or even estimated with a conventional Monte Carlo estimator, except in very special cases~\citep{rainforth2018nesting}.
In fact, it is extremely challenging and costly to derive \emph{any unbiased} estimate for it or its gradients.
To train $\pi_\phi$ with stochastic gradient methods, we will therefore introduce and optimize
\emph{lower bounds} on $\mathcal{I}_T(\pi_\phi)$, building on the ideas of \S~\ref{sec:cib}.

Equation~\eqref{eq:objective} shows that the objective function is the expected logarithm of a ratio of two terms. The first  is the likelihood of the history, $p(h_T|\theta,\pi)$, and can be directly evaluated using \eqref{eq:jlikelihood}.
The second term is an intractable marginal $p(h_T|\pi)$ that is different for each sample of the outer expectation
and must thus be estimated separately each time.

Given a sample $\theta_0,h_T\sim  p(\theta, h_T|\pi)$, we can perform this estimation by introducing 
$L$ independent \emph{contrastive} samples $\theta_{1:L}\sim p(\theta)$. We can then approximate the log-ratio in two different ways,
depending on whether or not we include $\theta_0$ in our estimate for $p(h_T|\pi)$:
\begin{align}
	\glo_L(\theta_{0:L},h_T) &= \log \frac{p(h_T|\theta_0,\pi)}{\frac{1}{L+1}\sum_{\ell=0}^L p(h_T|\theta_\ell,\pi)} \label{eq:gl} \\
	\gup_L(\theta_{0:L},h_T) &= \log \frac{p(h_T|\theta_0,\pi)}{\frac{1}{L}\sum_{\ell=1}^L p(h_T|\theta_\ell,\pi)}.\label{eq:fl}
\end{align}
These functions can both be evaluated by recomputing the likelihood of the history under each of the contrastive samples $\theta_{1:L}$.
We note that $\glo$ cannot exceed $\log(L+1)$, whereas $\gup$ is potentially unbounded (see Appendix~\ref{sec:proofs} for a proof).

We now show that using $\glo$ to approximate the integrand leads to a \emph{lower} bound on the overall objective $\mathcal{I}_T(\pi)$, whilst using $\gup$ leads to an \emph{upper} bound. During training, we focus on the lower bound, because it does not lead to unbounded ratio estimates and is therefore more numerically stable. We refer to this new lower bound as \emph{sequential PCE} (sPCE).
\begin{restatable}[Sequential PCE]{theorem}{seqace}
	\label{thm:seqace}
	For a design function $\pi$ and a number of contrastive samples $L\ge0$, let
	\begin{equation}
	\label{eq:sPCE_objective}
	\mathcal{L}_T(\pi,L) =
	\E_{p(\theta_{0},h_T|\pi)p(\theta_{1:L})}\left[ \glo_L(\theta_{0:L},h_T) \right]
	\end{equation}
	where $\glo_L(\theta_{0:L},h_T)$ is as per~\eqref{eq:gl}, and  $\theta_0,h_T \sim p(\theta,h_T|\pi)$, and $\theta_{1:L}\sim p(\theta)$ independently. 
	Given minor technical assumptions discussed in the proof, we have\footnote{$x_L\uparrow x$ means that $x_L$ is a monotonically increasing sequence in $L$ with limit $x$.}  
	\begin{equation}
	\mathcal{L}_T(\pi,L) \uparrow \mathcal{I}_T(\pi) \text{ as } L \to \infty
	\end{equation}
	at a rate $\mathcal{O}\left(L^{-1}\right)$.
\end{restatable}
The proof is presented in Appendix~\ref{sec:proofs}.
For evaluation purposes, it is helpful to pair sPCE with an upper bound, which we obtain by using $\gup$ as our estimate of the integrand
\begin{equation}
	\mathcal{U}_T(\pi,L) = \E_{p(\theta_{0},h_T|\pi)p(\theta_{1:L})}\left[ \gup_L(\theta_{0:L},h_T) \right].
\end{equation}
We refer to this bound as sequential Nested Monte Carlo (sNMC). 
Theorem~\ref{thm:seqnmc} in Appendix~\ref{sec:proofs} shows that $\mathcal{U}_T(\pi,L)$ satisfies complementary properties to $\mathcal{L}_T(\pi,L)$.
In particular, $\mathcal{L}_T(\pi,L) \le \mathcal{I}_T(\pi) \le \mathcal{U}_T(\pi,L)$ and both bounds become monotonically
tighter as $L$ increases, becoming exact as $L\to\infty$ at a rate $\mathcal{O}\left(1/L\right)$.  We can thus directly control the trade-off between bias in our objective and 
the computational cost of training. Note that increasing $L$ has no impact on the cost at deployment time.
Critically, as we will see in our experiments, we tend to only need relatively modest values of $L$ for $\mathcal{L}_T(\pi,L)$ to be an effective 
objective. 

If using a sufficiently large $L$ proves problematic (e.g.~our available training time is strictly limited), one can further tighten these
bounds for a fixed $L$ by introducing an amortized proposal, $q(\theta;h_T)$, for the contrastive samples $\theta_{1:L}$, rather than drawing them from the prior, as in \citet{foster2020unified}.
By appropriately adapting $\mathcal{L}_T(\pi,L)$, the proposal and the design network can then be trained simultaneously with a single unified objective, 
in a manner similar to a variational autoencoder~\citep{kingma2014auto}, allowing the bound itself to get tighter during training.
The resulting more general class of bounds are described in detail in Appendix~\ref{sec:app:bounds} and may offer further improvements for the DAD
approach. 
We focus on training with sPCE here in the interest of simplicity of both exposition and implementation.

\subsection{Gradient estimation}
\label{sec:gradientestimation}
The design network parameters $\phi$ can be optimized using a stochastic optimization scheme such as Adam~\citep{kingma2014adam}. Such methods require us to compute unbiased gradient estimates of the sPCE objective~\eqref{eq:sPCE_objective}.  Throughout, we assume that the design space $\Xi$ is continuous. 

We first consider the case when the observation space $\mathcal{Y}$ is also continuous and the likelihood $p(y|\theta,\xi)$ is reparametrizable.  This means that we can introduce random variables $\epsilon_{1:T}\sim p(\epsilon)$, which are independent of $\xi_{1:T}$ and $\theta_{0:L}$, such that $y_t = y(\theta_0,\xi_t,\epsilon_t)$. As we already have that $\xi_t = \pi_{\phi}(h_{t-1})$, we see that $h_{t}$ becomes a deterministic function of $h_{t-1}$ given $\epsilon_t$ and $\theta_0$. Under these assumptions we can take the gradient operator inside the expectation and apply the law of the unconscious statistician to write\footnote{We use $\partial a /\partial b$ and $da/db$ to represent the Jacobian matrices of partial and total derivatives respectively for vectors $a$ and $b$.}
\begin{equation}
	\dydx{\mathcal{L}_T}{\phi}  = \E_{p(\theta_{0:L})p(\epsilon_{1:T})}\left[ \dydx{}{\phi}\glo_L(\theta_{0:L},h_T)  \right].
	\label{eq:lossgrad}
\end{equation}
We can now construct unbiased gradient estimates by sampling from $p(\theta_{0:L})p(\epsilon_{1:T})$ and evaluating, $d \glo_L(\theta_{0:L},h_T)/ d\phi$. This gradient can be easily computed via an automatic differentiation framework \citep{baydin2018ad,pytorch}. 

For the case of discrete observations $y \in \mathcal{Y}$,  first note that given a policy $\pi_\phi$, the only randomness in the history $h_T$ comes from the observations $y_1, \dots, y_T$, since the designs are computed deterministically from past histories. 
One approach to computing the gradient of \eqref{eq:sPCE_objective} in this case is to sum over all possible histories $h_T$, integrating out the variables $y_{1:T}$, and take gradients with respect to $\phi$ to give
\begin{equation}
\dydx{\mathcal{L}_T}{\phi}  = \E\left[ \sum_{h_T} \dydx{}{\phi}  \Big(p(h_T|\theta_0) \glo_L(\theta_{0:L},h_T) \Big) \right],
\label{eq:total_enum}
\end{equation}
where the expectation is over $\theta_{0:L}\sim p(\theta)$. Unbiased gradient estimates can be computed using samples from the prior. Unfortunately, this gradient estimator has a computational cost $\mathcal{O}(|\mathcal{Y}|^T)$ and is therefore only applicable when both the number of experiments $T$ and the number of possible outcomes $|\mathcal{Y}|$ are relatively small. 

To deal with the cases when it is either impractical to enumerate all possible histories,  or $\mathcal{Y}$ is continuous but the  likelihood $\likelihood$ is non-reparametrizable, we propose using the score function gradient estimator,  which is also known as the REINFORCE estimator \citep{Williams1992}. The score function gradient, is given by
\begin{align}
\dydx{\mathcal{L}_T}{\phi}\!=\!\E &\left[ \!\left(  \log \frac{\likelihoodi{0}}{\sum_{\ell=0}^L \likelihoodi{\ell} } \right)\! \ddx{\phi} \log \likelihoodi{0} \right. \nonumber \\
& \quad \left. -\ddx{\phi}  \log  \sum_{\ell=0}^L \likelihoodi{\ell} \right] \label{eq:score_grad}
\end{align}
where the expectation is over $\theta_0,h_T \sim p(\theta,h_T|\pi)$ and $\theta_{1:L}\sim p(\theta)$, and unbiased estimates may again be obtained using samples. This gradient is amenable to the wide range of existing variance reduction methods such as control variates~\citep{tucker2017rebar,mohamed2020monte}. In our experiments, however, we found the standard score function gradient to be sufficiently low variance. For complete derivations of the gradients estimators we use, see Appendix~\ref{sec:app:method}.

\subsection{Architecture}
\label{sec:arch}
Finally, we discuss the deep learning architecture used for $\pi_\phi$. To allow efficient and effective training, we take into account a key permutation invariance of the BOED problem as highlighted by the following result (proved in Appendix~\ref{sec:proofs}). 

\begin{restatable}[Permutation invariance]{theorem}{permutation}
	\label{theorem:invariance}
	Consider a permutation $\sigma \in S_k$ acting on a history $h_k^1$, yielding
	$h_k^2 = (\xi_{\sigma(1)},y_{\sigma(1)}),...,(\xi_{\sigma(k)},y_{\sigma(k)})$.
	For all such $\sigma$, we have
	\begin{align*}
	\E \left[ \sum_{t=1}^{T} I_{h_{t-1}}(\xi_t) \middle| h_k = h^1_k\right]
	\!=
	\E \left[ \sum_{t=1}^{T} I_{h_{t-1}}(\xi_t) \middle| h_k = h^2_k\right]
	\end{align*}
	such that the EIG is unchanged under permutation. Further, the optimal policies starting in $h_k^1$ and $h_k^2$ are the same.
\end{restatable}
This permutation invariance is an important and well-studied property of many machine learning problems \citep{bloem2019probabilistic}.
The knowledge that a system exhibits permutation invariance can be exploited in neural architecture design to enable significant \emph{weight sharing}. 
One common approach is pooling~\citep{edwards2016towards,zaheer2017deep,garnelo2018conditional,garnelo2018neural}.   
This involves summing or otherwise combining representations of multiple inputs into a single representation that is invariant to their order.

Using this idea, we represent the history $h_t$ with a fixed dimensional representation that is formed by pooling representations of the distinct design-outcome pairs of the history
\begin{align}
	\label{eq:defnofr}
	R(h_t) :=  \sum\nolimits_{k=1}^t E_{\phi_1}(\xi_k,y_k),
\end{align}
where $E_{\phi_1}$ is a neural network \emph{encoder} with parameters $\phi_1$ to be learned. Note that this pooled representation is the same if we reorder the labels $1,...,t$. By convention, the sum of an empty sequence is 0. 

We then construct our design network 
 to make decisions based on the pooled representation $R(h_t)$ by setting $\pi_\phi(h_t) = F_{\phi_2}(R(h_t))$, where $F_{\phi_2}$ is a learned \emph{emitter} network. The trainable parameters  are $\phi=\{\phi_1, \phi_2\}$. By combining simple networks in a way that is sensitive to the permutation invariance of the problem, we facilitate parameter sharing in which the network $E_{\phi_1}$ is re-used for each input pair and for each time step $t$.  This results in significantly improved performance compared to networks that are forced to \emph{learn} the relevant symmetries of the problem. 

%% file: relwork.tex

\section{Related Work}
\label{sec:relwork}
Existing approaches to sequential BOED typically follow the path outlined in \S~\ref{sec:conventional}. The posterior inference performed at each stage of the conventional approach has been done using sequential Monte Carlo \citep{del2006sequential,drovandi2014sequential}, population Monte Carlo \citep{rainforth2017thesis}, variational inference \citep{foster2019variational,foster2020unified}, and Laplace approximation \citep{lewi2009sequential,long2013}.  

The estimation of the mutual information objective at each step has been performed by nested Monte Carlo \citep{myung2013,vincent2017}, variational bounds \citep{foster2019variational,foster2020unified}, Laplace approximation~\citep{lewi2009sequential}, ratio estimation \citep{kleinegesse2020sequential}, and hybrid methods \citep{senarathne2020laplace}.  The optimization over designs has been performed by Bayesian optimization \citep{foster2019variational,kleinegesse2020sequential}, interacting particle systems \citep{amzal2006bayesian}, simulated annealing \citep{muller2005simulation}, utilizing regret bounds \citep{zheng2020sequential}, or bandit methods \citep{rainforth2017thesis}.

There are approaches that simultaneously estimate the mutual information and optimize it, using a single stochastic gradient procedure. 
Examples include perturbation analysis \citep{huan2014gradient}, variational lower bounds \citep{foster2020unified}, or multi-level Monte Carlo \citep{goda2020unbiased}. 

Some recent work has focused specifically on models with intractable likelihoods \citep{Hainy2016likelihoodfree,kleinegesse2020mine, kleinegesse2020sequential}. Other work has sought to learn a non-myopic strategy focusing on specific tractable cases \citep{huan2016sequential,jiang2020binoculars}.

%% file: experiments.tex

\section{Experiments}
\label{sec:experiments}

\begin{figure}[t]
	\centering
	\includegraphics[width=0.46\textwidth]{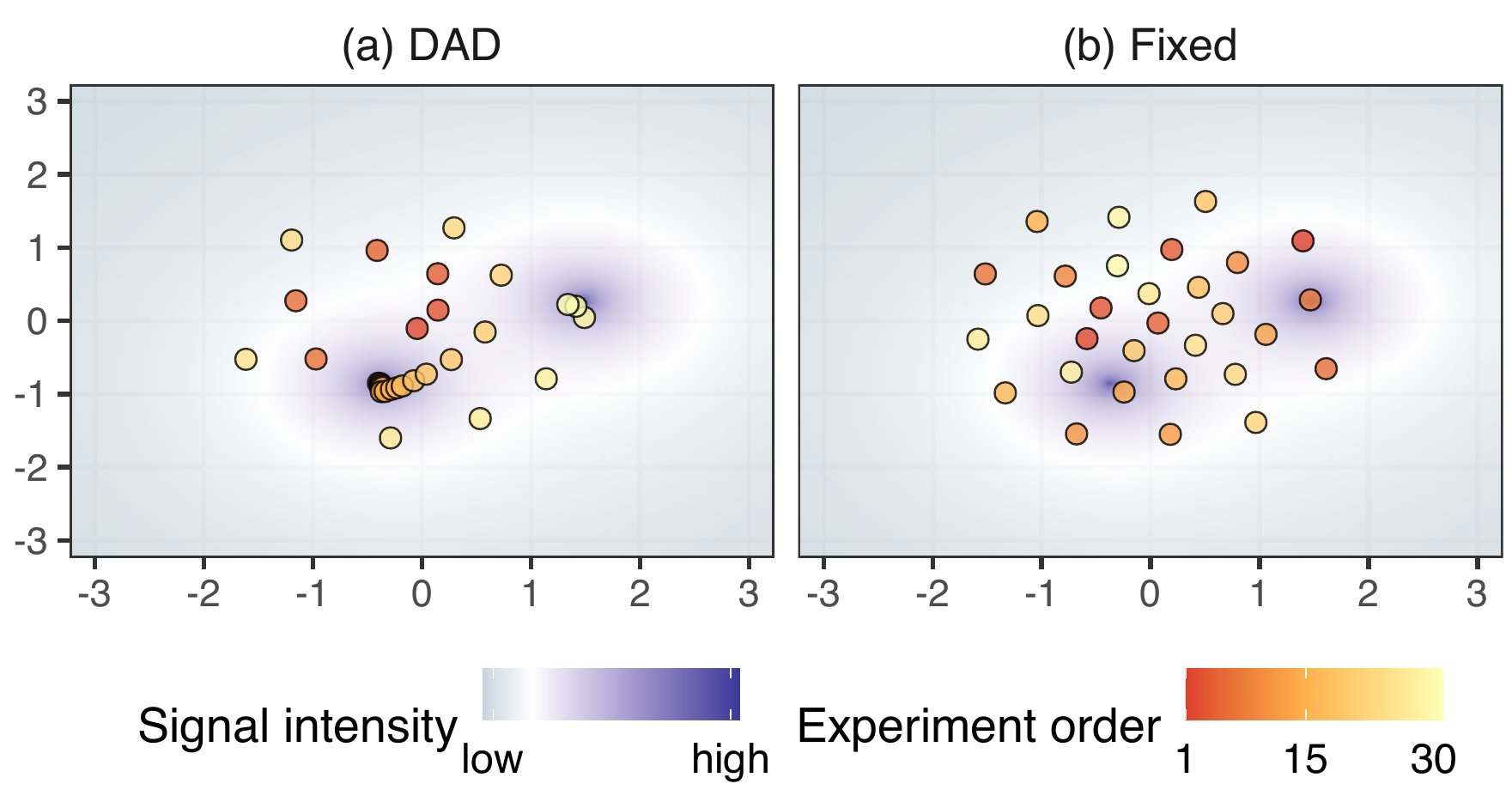}
\vspace{-5pt}
	\caption{An example of the designs learnt by (a) the DAD network and (b) the fixed baseline for a given $\theta$ sampled from the prior.}
	\label{fig:locfin_experiment_viz}
\end{figure}

We now compare \textbf{DAD} to a number of baselines across a range of experimental design problems.  
We implement DAD by extending PyTorch~\citep{pytorch} and Pyro~\citep{pyro} to provide an implementation that is abstracted from the specific problem. 
Code is publicly available at {\small \url{https://github.com/ae-foster/dad}}.

As our aim is to adapt designs in \emph{real-time}, we primarily compare to strategies that are fast at deployment time.  The simplest baseline is \textbf{random} design, which selects designs uniformly at random.
The \textbf{fixed} baseline completely ignores the opportunity for adaptation and uses static design to learn a fixed $\xi_1,...,\xi_T$ before the experiment.
We use the SG-BOED approach of \citet{foster2020unified} with the PCE bound to optimize the fixed design $\xi_{1:T}$.
We also compare to tailor-made heuristics for particular models as appropriate. 

Similarly to the notion of the amortization gap in amortized inference \citep{cremer2018suboptimality}, one might initially expect to a drop in performance of DAD compared to conventional (non-amortized) BOED methods that use the traditional iterative approach of \S~\ref{sec:conventional}. To assess this we also consider using the SG-BOED approach of~\citet{foster2020unified} in a traditional iterative manner to approximate $\pi_s$, referring to this as the \textbf{variational} baseline, noting this requires significant run-time computation.
We also look at several iterative BOED baselines that are specifically tailored to the examples that we choose \citep{vincent2017,kleinegesse2020sequential}.  Perhaps surprisingly, we find that DAD is not only competitive compared to these non-amortized methods, but often outperforms them. This is discussed in \S~\ref{sec:discussion}.

The first performance metric that we focus on is total EIG, $\mathcal{I}_T(\pi)$.  When no direct estimate of $\mathcal{I}_T(\pi)$ is available, we estimate both the sPCE lower bound and sNMC upper bound.  We also present the standard error to indicate how the performance varies between different experiment realizations (rollouts).  We further consider the deployment time (i.e.~the time to run the experiment itself, after pre-training); a critical metric for our aims. Full experiment details are given in Appendix~\ref{sec:exp-appendix}.

\input{experiments_loc_finding}

\input{experiments_hyperbolic_temporal}
\input{experiments_death_process}

%% file: experiments_loc_finding.tex

\subsection{Location finding in 2D}\label{sec:experiments_locfin}
\begin{table}[t]
	\begin{tabular}{lrr}
		Method  &  Lower bound, $\mathcal{L}_{30}$ &  Upper bound, $\mathcal{U}_{30}$   \\
		\hline
		Random         & 8.303 $\pm$ 0.043 &  8.322 $\pm$ 0.045 \\
		Fixed              & 8.838 $\pm$ 0.039 &   8.914 $\pm$ 0.038  \\
		\textbf{DAD}  & \textbf{10.926} $\pm$ \textbf{0.036}  & \textbf{12.382} $\pm$ \textbf{0.095}     \\
		\hline
		Variational     & 8.776 $\pm$ 0.143 &  9.064 $\pm$ 0.187
	\end{tabular}
	\vspace{-4pt}
	\caption{Upper and lower bounds on the total EIG, $\mathcal{I}_{30}(\pi)$, for the location finding experiment.  Errors indicate $\pm1$ s.e.~estimated over 256 (variational) or 2048 (others) rollouts.
		\vspace{-8pt}
		}
	\label{tab:locfin_T30_results}
\end{table}

\begin{figure}[t]
	\centering
	\includegraphics[width=0.46\textwidth]{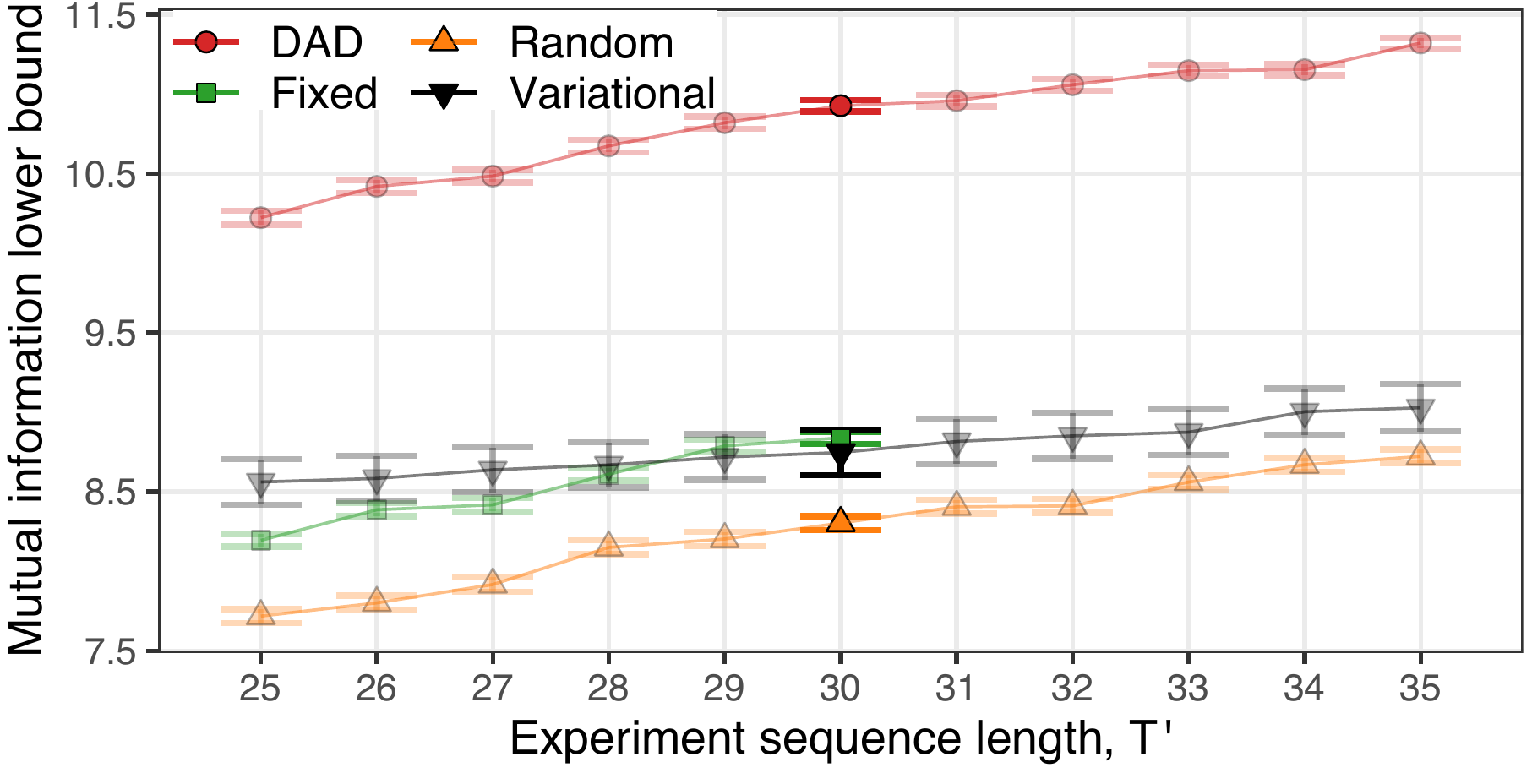}
	\vspace{-5pt}
	\caption{Generalizing sequence length for the location finding experiment.
		The DAD network and the fixed strategy were trained to perform $T=30$ experiments, whilst other strategies do not require pre-training.
		The fixed strategy cannot be generalized to sequences longer than its training regime.
		We present sPCE estimates with error bars computed as in Table~\ref{tab:locfin_T30_results}.
	}
	\label{fig:locfin_2D_2S}
\end{figure}

Inspired by the acoustic energy attenuation model of~\citet{Sheng2005}, we consider the problem of finding the locations of multiple hidden sources which each emits a signal whose intensity attenuates according to the inverse-square law.
The \emph{total intensity} is a superposition of these signals. The design problem is to choose where to make observations of the total signal to learn the locations of the sources.

We train a DAD network to perform $T=30$ experiments with $K=2$ sources. The designs learned by DAD are visualized in Figure \ref{fig:locfin_experiment_viz}(a). Here our network learns a complex strategy that initially explores in a spiral pattern. 
Once it detects a strong signal, multiple experiments are performed close together to refine knowledge of that location (note the high density of evaluations near the sources).
The fixed design strategy, displayed in Figure \ref{fig:locfin_experiment_viz}(b) must choose all design locations up front, leading to an evenly dispersed strategy that cannot ``hone in'' on the critical areas, thus gathering less information. 

Table \ref{tab:locfin_T30_results} reports upper and lower bounds on $\mathcal{I}_T(\pi)$ for each strategy and confirms that DAD significantly outperforms all the considered baselines.
DAD is also orders of magnitude faster to deploy than the  variational baseline, the other adaptive method, with DAD taking $0.0474\pm 0.0003$ secs to make all 30 design decisions on a lightweight CPU, compared to $8963$ secs for the variational method.

\paragraph{Varying the Design Horizon} In practical situations the exact number of experiments to perform may be unknown.  Figure  \ref{fig:locfin_2D_2S} indicates that our DAD network that is pretrained to perform $T=30$ experiments can generalize well to perform $T'\neq30$ experiments at deployment time,  still outperforming the baselines,  indicating that DAD is robust to the length of training sequences. 

\paragraph{Training Stability} To assess the stability between different training runs, we trained 16 different DAD networks. Computing the mean and standard error of the lower bound on $\mathcal{I}_T(\pi)$ over these 16 runs gave $10.91 \pm 0.014$, and the matching upper bounds were $12.47 \pm 0.046$. We see that the variance across different training seeds is modest, indicating that DAD reaches designs of a similar quality each time. Comparing with Table~\ref{tab:locfin_T30_results}, we see that the natural variability across rollouts (i.e.~different $\theta$) with a single DAD network tends to be larger than the variance between the average performance of different DAD networks.

%% file: experiments_hyperbolic_temporal.tex

\begin{table}[t]
	\centering
	\begin{tabular}{lr}
		Method & Deployment time (s) \\
		\hline
		\citet{frye2016measuring} & 0.0902 $\pm$ 0.0003 \\
		\citet{kirby2009one} & N/A \\
		Fixed & N/A \\
		DAD & 0.0901 $\pm$ 0.0007 \\
		\hline
		Badapted & 25.2679 $\pm$ 0.1854 
	\end{tabular}
	\caption{Deployment times for Hyperbolic Temporal Discounting methods.
		We present the total design time for $T=20$ questions, taking the mean and $\pm1$ s.e.~over 10 realizations. 
		Tests were conducted on a lightweight CPU (see Appendix~\ref{sec:exp-appendix}).
	\vspace{-3pt}}
	\label{tab:hyperbolic_time}
\end{table}

\begin{table}[t]
	\centering	
	\begin{tabular}{lrr}
		Method & Lower bound & Upper bound \\
		\hline
		\citet{frye2016measuring} & 3.500 $\pm$ 0.029 & 3.513 $\pm$ 0.029 \\
		\citet{kirby2009one} & 1.861 $\pm$ 0.008 & 1.864 $\pm$ 0.009 \\
		Fixed & 2.518 $\pm$ 0.007 & 2.524 $\pm$ 0.007 \\
		\textbf{DAD} & \textbf{5.021 $\pm$ 0.013} & \textbf{5.123 $\pm$ 0.015}\\
		\hline
		Badapted &  4.454 $\pm$ 0.016 & 4.536 $\pm$ 0.018
	\end{tabular}
	
	\caption{Final lower and upper bounds on the total information $\mathcal{I}_T(\pi)$ for the Hyperbolic Temporal Discounting experiment. The bounds are finite sample estimates of $\mathcal{L}_T(\pi,L)$ and $\mathcal{U}_T(\pi,L)$ with $L=5000$. The errors indicate $\pm1$ s.e.~over the sampled histories.
    \vspace{-8pt}}
	\label{tab:hyperbolic_eig}
\end{table}

\begin{figure}[t]
	\centering
	\includegraphics[width=0.46\textwidth]{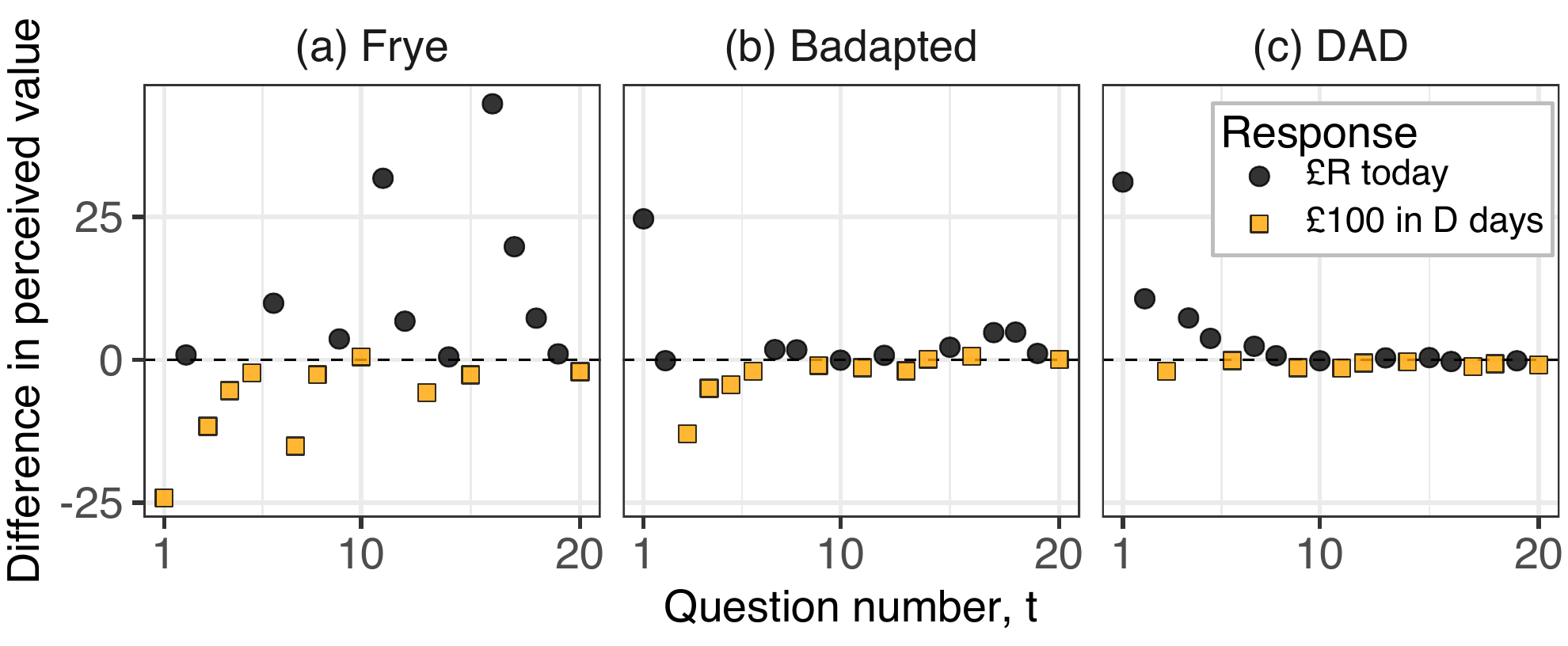}
	\vspace{-5pt}
	\caption{ 
	An example of the designs learnt by two of the problem-specific baselines and DAD. We plot the difference in perceived value of the two propositions ``\textsterling$R$ today'' and ``\textsterling100 in $D$ days'' for a certain participant, represented by a specific value of the latent variable $\theta$. A difference of 0 indicates that the participant is indifferent between the two offers.
	\vspace{-6pt}
	}
	\label{fig:hyperbolic_vis}
\end{figure}

\subsection{Hyperbolic temporal discounting}
\label{sec:hyperbolic}
In psychology, temporal discounting is the phenomenon that the utility people attribute to a reward typically decreases the longer they have to wait to receive it \citep{critchfield2001temporal,green2004discounting}.
For example, a participant might be willing to trade \textsterling90 today for \textsterling100 in a month's time, but not for \textsterling100 in a year.
A common parametric model for temporal discounting in humans is the hyperbolic model \citep{mazur1987adjusting}; we study a specific form of this model proposed by \citet{vincent2016hierarchical}. 

We design a sequence of $T=20$ experiments, each taking the form of a binary question ``Would you prefer \textsterling$R$ today, or \textsterling100 in $D$ days?'' with design $\xi = (R,D)$ that must be chosen at each stage.  As real applications of this model would involve human participants, the available time to choose designs is strictly limited. We consider DAD, the aforementioned fixed design policy, and strategies that have been used specifically for experiments of this kind: \citet{kirby2009one}, a human constructed fixed set of designs; \citet{frye2016measuring}, a problem-specific adaptive strategy; and \citet{vincent2017}, a partially customized sequential BOED method, called Badapted, that uses population Monte Carlo~\citep{cappe2004population} to approximate the posterior distribution at each step and a bandit approach to optimize the EIG over possible designs. 

We begin by investigating the time required to deploy each of these methods. As shown in Table~\ref{tab:hyperbolic_time}, the non-amortized Badapted method takes the longest time, while for DAD, the total deployment time is less than $0.1$ seconds---totally imperceptible to a participant. 

Table~\ref{tab:hyperbolic_eig} shows the performance of each method. We see that DAD performs best, surpassing bespoke design methods that have been proposed for this problem, including Badapted which has a considerably larger computation budget.  Figure \ref{fig:hyperbolic_vis} demonstrates how the designs learnt by DAD compare qualitatively with the two most competitive problem-specific baselines.
As with Badapted, DAD designs rapidly cluster near the indifference point.

This experiment demonstrates that DAD can successfully amortize the process of experimental design in a real application setting. It outperforms some of the most successful non-amortized and highly problem-specific approaches with a fraction of the cost during the real experiment.   

%% file: experiments_death_process.tex

\subsection{Death process}
\label{sec:death}
We conclude with an example from epidemiology \citep{cook2008optimal} in which healthy individuals become infected at rate $\theta$. The design problem  is to choose observations times $\xi > 0$ at which to observe the number of infected individuals: we select $T=4$ designs sequentially with an independent stochastic process observed at each iteration. We compare to our fixed and variational baselines, along with the adaptive SeqBED approach of~\citet{kleinegesse2020sequential}.

First, we examine the compute time required to deploy each method for a single run of the sequential experiment. The times illustrated in Table~\ref{tab:death_eig} show that the adaptive strategy learned by DAD can be deployed in under $0.01$ seconds, many orders of magnitude faster than the non-amortized methods, with SeqBED taking hours for \emph{one} rollout.

Next, we estimate the objective $\mathcal{I}_T(\pi)$  by averaging the information gain over simulated rollouts. The results in  Table~\ref{tab:death_eig} reveal that DAD designs are superior to both fixed design and variational adaptive design, tending to uncover more information about the latent $\theta$ across many possible experimental trajectories.
For comparison with SeqBED, we were unable to perform sufficient rollouts to obtain a high quality estimate of $\mathcal{I}_T(\pi)$. Instead, we conducted a single rollout of each method with $\theta=1.5$ fixed. The resulting information gains for this one rollout were: 1.590 (SeqBED), 1.719 (Variational), 1.678 (Fixed),  \textbf{1.779 (DAD)}.

\begin{table}[t]
	\centering
	\begin{tabular}{lrr}
		Method & Deployment time (s) & $\mathcal{I}_T(\pi)$  \\
		\hline
		Fixed & N/A 
		& 2.023 $\pm$ 0.007  \\
		\textbf{DAD} & 0.0051 $\pm$ 12\% &  \textbf{2.113} $\pm$ \textbf{0.008} \\
		\hline
		Variational & 1935.0\phantom{000} $\pm$ \phantom{0}2\% & 2.076 $\pm$ 0.034  \\
		SeqBED* & 25911.0\phantom{000} \phantom{$\pm$ 04\%} & 1.590 \phantom{$\pm$ 0.034} 
	\end{tabular}
	\caption{Total EIG $\mathcal{I}_T(\pi)$ and deployment times for the Death Process. We present the EIG $\pm1$ s.e. over 10,000 rollouts (fixed and DAD), 500 rollouts (variational) or *1 rollout (SeqBED). The IG can be efficiently evaluated in this case (see Appendix~\ref{sec:exp-appendix}). Runtimes computed as per Table~\ref{tab:hyperbolic_time}. 
	\vspace{-5pt}}
	\label{tab:death_eig}
\end{table}

%% file: discussion.tex

\vspace{-2pt}
\section{Discussion}\label{sec:discussion}
\vspace{-2pt}

In this paper we introduced DAD---a new method utilizing the power of deep learning to amortize the cost of sequential BOED and allow adaptive experiments to be run in real time. In all experiments DAD performed significantly better than baselines with a comparable deployment time.  Further, DAD showed competitive performance against conventional BOED approaches that do not use amortization, but make costly computations at each stage of the experiment.

Surprisingly, we found DAD was often able to outperform these non-amortized approaches despite using a tiny fraction of the resources at deployment time. We suggest two reasons for this. Firstly, conventional methods must approximate the posterior $p(\theta|h_t)$ at each stage.
If this approximation is poor, the resulting design optimization will yield poor results regardless of the EIG optimization approach chosen.
Careful tuning of the posterior approximation could alleviate this, but would increase computational time further and it is
difficult to do this in the required automated manner. DAD sidesteps this problem altogether by eliminating the need for directly approximating a posterior distribution.

Secondly, the policy learnt by DAD has the potential to be \emph{non-myopic}: it does not choose a design that is optimal for the current experiment in isolation, but takes into account the fact that there are more experiments to be performed in the future. 
We can see this in practice in a simple experiment using the location finding example with one source in 1D with prior $\theta\sim N(0,1)$ and with $T=2$ steps. This setting is simple enough to compute the \emph{exact} one-step optimal design via numerical integration. 
Figure \ref{fig:locfin_T2_comparison} [Left] shows the design function learnt by DAD alongside the exact optimal myopic design. The optimal myopic strategy for $t=1$ is to sample at the prior mean $\xi_1=0$.  At time $t=2$ the myopic strategy  selects a positive or negative design with equal probability.  In contrast,  the policy learnt by DAD is to sample at $\xi_1\approx-0.4$, which does not optimize the EIG for $T=1$ in isolation, but leads to a better \emph{overall} design strategy that focuses on searching the positive regime $\xi_2>\xi_1$ in the second experiment. Figure~\ref{fig:locfin_T2_comparison} [Right] confirms that the policy learned by DAD achieves higher total EIG from the two step experiment than the \emph{exact} myopic approach.

\paragraph{Limitations and Future Work} The present form of DAD still possesses some restrictions that future work might look to address.
Firstly, it requires the likelihood model to be \emph{explicit}, i.e.~that we can evaluate the density $p(y_t|\theta,\xi_t)$.
Secondly, it requires the experiments to be conditionally independent given $\theta$, i.e. $p(y_{1:T}|\theta,\xi_{1:T})=\prod_{t=1}^{T} p(y_t|\theta,\xi_t)$, which may not be the case for, e.g.~time series models.
Thirdly, it requires the designs themselves, $\xi_t$, to be continuous to allow for gradient-based optimization.
On another note, DAD's use of a policy to make design decisions establishes a critical link between experimental design and model-based reinforcement learning~\citep{sekar2020planning}.
Though DAD is distinct in several important ways (e.g.~the lack of observed rewards), investigating these links further might provide an interesting avenue for future work.

\begin{figure}[t]
	\centering
	\begin{subfigure}[b]{0.48\linewidth}
		\centering\includegraphics[width=1\textwidth]{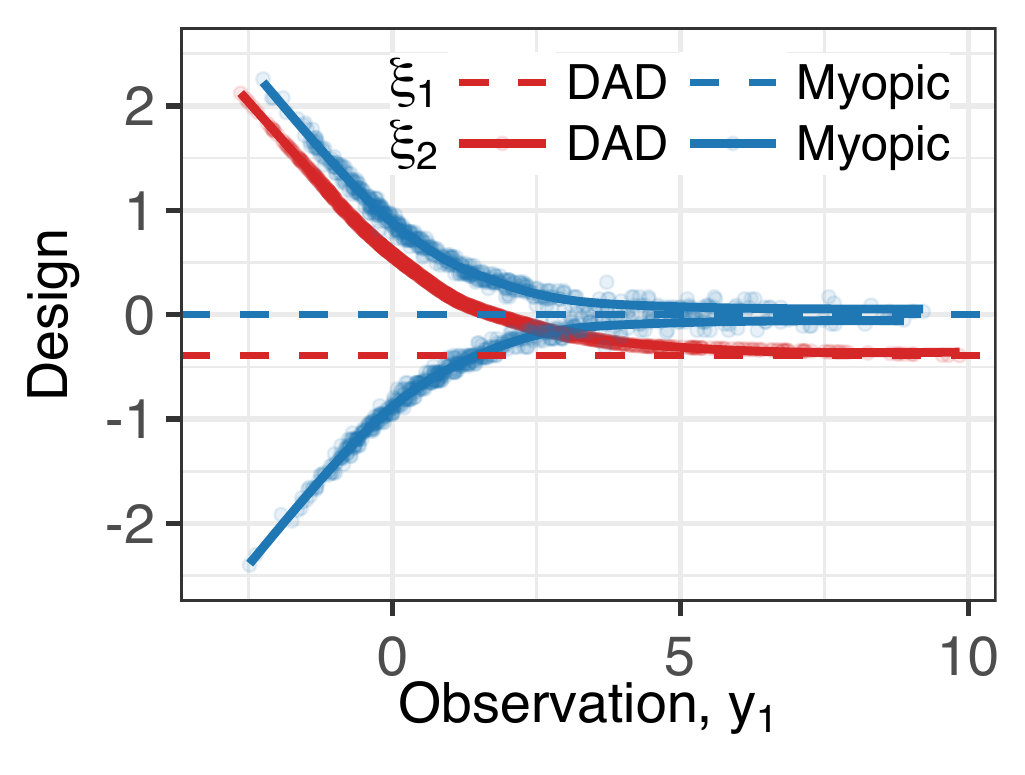}
	\end{subfigure}
	\begin{subfigure}[b]{0.48\linewidth}
		\centering\includegraphics[width=1\textwidth]{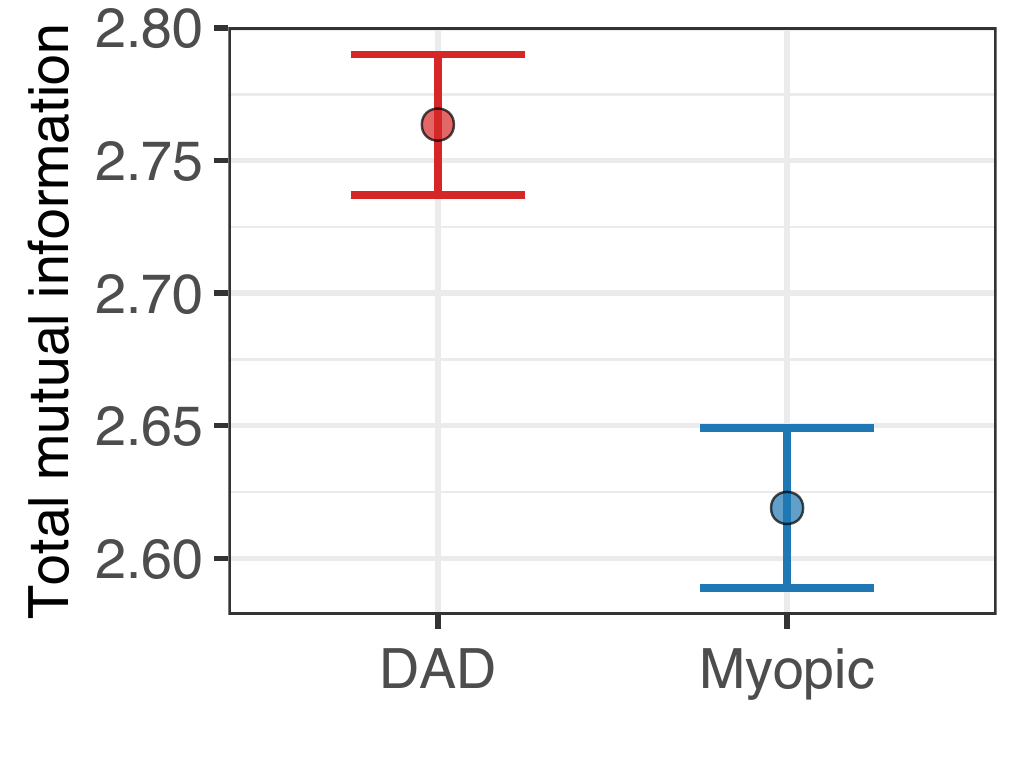}
	\end{subfigure}
	\vspace{-6pt}
	\caption{1D location finding with 1 source, $T=2$.   [Left] the design function, dashed lines correspond to the first design $\xi_1$, which is independent of $y_1$. [Right] $\mathcal{I}_2(\pi)$, the total EIG $\pm1$ s.e.
	\vspace{-8pt}}
	\label{fig:locfin_T2_comparison}
\end{figure}

\paragraph{Conclusions} DAD represents a new conception of adaptive experimentation that focuses on learning a design \emph{policy} network offline, then deploying it during the live experiment to quickly make adaptive design decisions.
This marks a departure from the well-worn path of myopic adaptive BOED (Sec.~\ref{sec:background}), eliminating the need to estimate intermediate posterior distributions or optimize over designs during the live experiment itself; it represents the first approach to allow adaptive BOED to be run in real-time for general problems.
As such, we believe it
may be beneficial to practitioners in a number of fields, from online surveys to clinical trials.

%% file: acknowledgements.tex

\section*{Acknowledgements}

AF gratefully acknowledges funding from EPSRC grant no.~EP/N509711/1. DRI is supported by EPSRC through the Modern Statistics and Statistical Machine Learning (StatML) CDT programme, grant no.~EP/S023151/1.
AF would like to thank Martin Jankowiak and Adam Golinski for helpful discussions about amortizing BOED.

%% file: appendix.tex

\onecolumn

\input{appendix_proofs}

\input{appendix_gradients}

\input{appendix_locfin}
\input{appendix_death_temporal}

%% file: appendix_proofs.tex

\section{Proofs}
\label{sec:proofs}
We begin by showing that $\glo_L(\theta_{0:L},h_T)$  from equation \eqref{eq:gl} is bounded by $\log(L+1)$ and that $\gup_L(\theta_{0:L},h_T)$ from equation \eqref{eq:fl} can potentially be unbounded. For the former
\begin{align}
\glo_L(\theta_{0:L},h_T)& = \log \frac{p(h_T|\theta_0, \pi)}{\frac{1}{L+1}\sum_{\ell=0}^L p(h_T|\theta_\ell, \pi)} \\
&= \log \frac{p(h_T|\theta_0, \pi)}{p(h_T|\theta_0, \pi) + \sum_{\ell =1}^L p(h_T|\theta_\ell, \pi)  }+ \log (L+1) \\
&\leq \log( 1 ) + \log (L+1).
\end{align}
For the latter we have
\begin{align*}
\gup_L(\theta_{0:L},h_T)=\log \frac{p(h_T|\theta_0, \pi)}{\frac{1}{L}\sum_{\ell=1}^L p(h_T|\theta_\ell, \pi)} &\to +\infty \textup{ as } \max_{1\le \ell \le L} p(h_T|\theta_\ell,\pi)\to 0 \textup{ with } p(h_T|\theta_0,\pi) \textup{ held constant.}
\end{align*}

Next we present proofs for all Theorems in the main paper, with each restated for convenience.
\terminal*
\begin{proof}
We begin by rewriting $I_{h_{t-1}}$ in terms of the information gain. This closely mimics the development that we presented in Section~\ref{sec:background}. By repeated appplication of Bayes Theorem we have
\begin{align}
	I_{h_{t-1}}(\xi_t) &= \E_{p(\theta|h_{t-1})p(y_t|\theta,\xi_t)}\left[ \log \frac{p(y_t|\theta,\xi_t)}{p(y_t|h_{t-1},\xi_t)} \right] \\
	&=\E_{p(\theta|h_{t-1})p(y_t|\theta,\xi_t)}\left[ \log \frac{p(\theta|h_{t-1})p(y_t|\theta,\xi_t)}{p(\theta|h_{t-1})p(y_t|h_{t-1},\xi_t)} \right] \\
	&=\E_{p(\theta|h_{t-1})p(y_t|\theta,\xi_t)}\left[ \log \frac{p(\theta|h_{t-1},\xi_t,y_t)}{p(\theta|h_{t-1})} \right] \\
	&=\E_{p(\theta|h_{t-1})}\left[ - \log p(\theta|h_{t-1}) \right] + \E_{p(y_t,\theta|\xi_t,h_{t-1})}\left[ \log p(\theta|h_{t-1},\xi_t,y_t) \right] \\
	&=\E_{p(\theta|h_{t-1})}\left[ - \log p(\theta|h_{t-1}) \right] + \E_{p(y_t|\xi_t,h_{t-1})p(\theta|h_{t-1},\xi_t,y_t)}\left[ \log p(\theta|h_{t-1},\xi_t,y_t) \right] \\
	&=\E_{p(y_t|\xi_t,h_{t-1})}\left[ \,H[\,p(\theta|h_{t-1})\,] - H[\,p(\theta|h_{t-1},\xi_t,y_t)\,]\, \right].
\end{align}
Now noting that each $I_{h_{t-1}}(\xi_t)$ is completely determined by $h_{t-1}$ and $\pi$ (in particular noting that $\xi_t$ is deterministic given these, while $\theta$ is already marginalized out in each $I_{h_{t-1}}(\xi_t)$), we can write
\begin{align}
	\mathcal{I}_T(\pi) &= 
	 \E_{p(h_T|\pi)}\left[ \sum_{t=1}^T I_{h_{t-1}}(\xi_t) \right] \\
	 &= \sum_{t=1}^T  \E_{p(h_{t-1}|\pi)}\left[ I_{h_{t-1}}(\xi_t) \right] 
\intertext{and substituting in our earlier formulation for $I_{h_{t-1}}(\xi_t)$}
	&= \sum_{t=1}^T \E_{p(h_{t-1}|\pi)}\left[ \E_{p(y_t|\xi_t,h_{t-1})}\left[ \,H[\,p(\theta|h_{t-1})\,] - H[\,p(\theta|h_{t-1},\xi_t,y_t)\,]\, \right] \right].
\displaybreak[0]
\intertext{We now observe that we can write $h_t = h_{t-1} \cup \{(\xi_t,y_t)\}$, which allows us to rewrite this as}
	&= \sum_{t=1}^T \E_{p(h_{t}|\pi)}\left[  \,H[\,p(\theta|h_{t-1})\,] - H[\,p(\theta|h_{t})\,]\, \right] \\
	&=\sum_{t=1}^T \E_{p(h_{T}|\pi)}\left[  \,H[\,p(\theta|h_{t-1})\,] - H[\,p(\theta|h_{t})\,]\, \right] \\
	&= \E_{p(h_T|\pi)}\left[ \sum_{t=1}^T H[\,p(\theta|h_{t-1})\,] - H[\,p(\theta|h_{t})\,] \right]\\
	&=\E_{p(h_T|\pi)}\left[ H[\,p(\theta)\,] - H[\,p(\theta|h_{T})\,] \right],
\displaybreak[0]
\intertext{where the last line follows from the fact that we have a telescopic sum. To complete the proof, we rearrange this as}
	&= \E_{p(\theta,h_T|\pi)}\left[ \log p(\theta|h_T) - \log p(\theta) \right] \\
	&= \E_{p(\theta)p(h_T|\theta,\pi)}\left[ \log \frac{p(\theta)p(h_T|\theta,\pi)}{p(h_T|\pi)} - \log p(\theta) \right] \\
	&= \E_{p(\theta)p(h_T|\theta,\pi)}\left[ \log p(h_T|\theta,\pi) - \log p(h_T|\pi) \right]
\end{align}
as required.
\end{proof}

\seqace*
\begin{proof}
We first show that $ \mathcal{L}_T(\pi,L)$ is a lower bound on $\mathcal{I}_T(\pi)$:
\begin{align}
 \mathcal{I}_T(\pi) - \mathcal{L}_T(\pi,L) &= \E_{p(\theta_0, h_T| \pi)}\left[ \log \frac{p(h_T|\theta_0,\pi)}{p(h_T|\pi)} \right]  - \E_{p(\theta_0, h_T|\pi)} \E_{p(\theta_{1:L})} \left[ \log \frac{p(h_T|\theta_0,\pi)}{\frac{1}{L+1}\sum_{\ell=0}^L p(h_T|\theta_\ell,\pi) } \right] \\
&= \E_{p(\theta_0, h_T|\pi)} \E_{p(\theta_{1:L})} \left[ \log \frac{\frac{1}{L+1} \sum_{\ell=0}^Lp(h_T|\theta_\ell, \pi)}{p(h_T|\pi)} \right] \label{eq:thm2_diff} \\
&= \E_{p(\theta_0, h_T|\pi)} \E_{p(\theta_{1:L})} \left[ \log \left(\frac{1}{L+1}\sum_{\ell=0}^L \frac{p(\theta_\ell|h_T)}{p(\theta_\ell)} \right)\right] \\
\intertext{now introducing the shorthand $p(\theta_{0:L}^{-\ell}):=p\left(\theta_{0:L\backslash\{\ell\}}\right)=\prod_{j=0,j\neq \ell}^L p(\theta_j)$,}
 &= \E_{p(\theta_0, h_T|\pi)} \E_{p(\theta_{1:L})} \left[ \log \frac{ \frac{1}{L+1}\sum_{\ell=0}^Lp(\theta_\ell|h_T) p(\theta_{0:L}^{-\ell}) }{ p(\theta_{0:L})} \right]. \\
 \intertext{Now by the systemtry on term in side the log, we see that this expectation would be the same if it were instead taken over $p(\theta_i, h_T|\pi )p(\theta_{0:L}^{-i})$ for any $i\in\{0,\dots,L\}$ (with $i=0$ giving the original form).  Furthermore, the result is unchanged if we take the expectation over the mixture distribution $\frac{1}{L+1} \sum_{i=0}^L p(\theta_i, h_T|\pi )p(\theta_{0:L}^{-i})=p(h_T|\pi)\frac{1}{L+1} \sum_{i=0}^L p(\theta_i|h_T) p(\theta_{0:L}^{-i})$
 	and thus we have}
&= \E_{p(h_T|\pi)} \E_{\frac{1}{L+1} \sum_{i=0}^L p(\theta_i|h_T) p(\theta_{0:L}^{-i})} \left[ \log \frac{  \frac{1}{L+1}\sum_{\ell=0}^Lp(\theta_\ell|h_T) p(\theta_{0:L}^{-\ell})  }{p(\theta_{0:L}) } \right] \\
&=\E_{p(h_T|\pi)} \big[ \textnormal{KL}\left( \tilde{p}(\theta_{0:L} |h_T)  || p(\theta_{0:L}) \right)  \big]  \label{eq:thm2_KL}
\end{align}
where 
$\tilde{p}(\theta_{0:L} |h_T) =  \frac{1}{L+1}\sum_{\ell=0}^Lp(\theta_\ell|h_T) p(\theta_{0:L}^{-\ell})$, which is indeed a distribution since
\begin{equation}
\int \tilde{p}(\theta_{0:L} |h_T) d\theta_{0:L} = \frac{1}{L+1}\sum_{\ell=0}^L\left(  \int p(\theta_\ell|h_T) d\theta_\ell \cdot \int p(\theta_{0:L}^{-\ell}) d\theta_{0:L}^{-\ell} \right) = 1.
\label{eq:ptilde_integral}
\end{equation}
Now by Gibbs' inequality the expected KL in~\eqref{eq:thm2_KL} must be non-negative,  establishing $ \mathcal{I}_T(\pi) - \mathcal{L}_T(\pi,L)\ge 0$ and thus $\mathcal{I}_T(\pi) \geq \mathcal{L}_T(\pi,L)$ as required.

We next show monotonicity in $L$, i.e.    $ \mathcal{L}_T(\pi,L_2) \geq \mathcal{L}_T(\pi,L_1) $ for $L_2 \geq L_1 \geq 0$, using similar argument as above
\begin{align}
 \mathcal{L}_T(\pi,L_2) - \mathcal{L}_T(\pi,L_1) &=     \E_{p(\theta_0, h_T|\pi)} \E_{p(\theta_{1:L_2})} \left[ \log \frac{   \frac{1}{L_1+1}\sum_{i=0}^{L_1} p(h_T|\theta_i,\pi)    }{     \frac{1}{L_2+1}\sum_{j=0}^{L_2} p(h_T|\theta_j,\pi) } \right]    \\
 &= \E_{p(\theta_0, h_T|\pi)} \E_{p(\theta_{1:L_2})} \left[ \log \frac{   \frac{1}{L_1+1}\sum_{i=0}^{L_1} \big( p(\theta_i| h_T) / p(\theta_i )  \big) }{     \frac{1}{L_2+1}\sum_{j=0}^{L_2} \big(p(\theta_j|h_T) / p(\theta_j) \big) } \right] \\
 & = \E_{p(\theta_0, h_T|\pi)} \E_{p(\theta_{1:L_2})} \left[ \log \frac{   \frac{1}{L_1+1}\sum_{i=0}^{L_1} \big( p(\theta_i| h_T)  p(\theta_{0:L_1}^{-i} ) \big) /  p(\theta_{0:L_1})  }{     \frac{1}{L_2+1}\sum_{j=0}^{L_2} \big(p(\theta_j|h_T)   p(\theta_{0:L_2}^{-j} ) \big) / p(\theta_{0:L_2} ) } \right]\\
 &= \E_{p(\theta_0, h_T|\pi)} \E_{p(\theta_{1:L_2})} \left[ \log \frac{   \frac{1}{L_1+1}\sum_{i=0}^{L_1} p(\theta_i| h_T)  p(\theta_{0:L_2}^{-i} )   }{     \frac{1}{L_2+1}\sum_{j=0}^{L_2} p(\theta_j|h_T)   p(\theta_{0:L_2}^{-j} )  } \right] \\
 &= \E_{p(h_T|\pi)} \E_{ \frac{1}{L+1} \sum_{\ell=0}^{L_1} p(\theta_\ell | h_T)p(\theta_{0:L_2}^{-\ell})} \left[ \log \frac{   \frac{1}{L_1+1}\sum_{i=0}^{L_1} p(\theta_i| h_T)  p(\theta_{0:L_2}^{-i} )   }{     \frac{1}{L_2+1}\sum_{j=0}^{L_2} p(\theta_j|h_T)   p(\theta_{0:L_2}^{-j} )  } \right] \label{eq:thm2_monotonicity_monster} \\ 
 &=  \E_{p(h_T|\pi)} \big[ \textnormal{KL}(\tilde{p}_1 || \tilde{p}_2) \big] \geq 0
\end{align}
where $\tilde{p}_1$ and $\tilde{p}_2$ are, respectively, the distributions in the numerator and denominator in \eqref{eq:thm2_monotonicity_monster}.
The result then again follows by Gibbs' inequality.

Next we show $\mathcal{L}_T(\pi,L) \rightarrow \mathcal{I}_T(\pi) \text{ as } L \to \infty$. First, note that the denominator in~\eqref{eq:sPCE_objective}, $\frac{1}{L+1}\sum_{\ell=0}^L p(h_T|\theta_\ell,\pi)$,  is a consistent estimator of the marginal $p(h_T|\pi)$, since $\frac{1}{L+1} p(h_T|\theta_0,\pi) \rightarrow 0$, and by the Strong Law of Large Numbers
\begin{align}
\frac{1}{L+1}\sum_{\ell=1}^L p(h_T|\theta_\ell,\pi) = 
\frac{L}{L+1}\cdot \frac{1}{L}\sum_{\ell=1}^L p(h_T|\theta_\ell,\pi)  \xrightarrow[]{\text{a.s. }} \E_{p(\theta) } \left[ p(h_T|\theta,\pi)  \right] = p(h_T|\pi).
\end{align}
Now from~\eqref{eq:thm2_diff} we also have that
\begin{align}
	\label{eq:thm2_diff_2}
	\mathcal{I}_T(\pi) - \mathcal{L}_T(\pi,L) &= \E_{p(\theta_0, h_T|\pi)} \E_{p(\theta_{1:L})} \left[ \log \frac{\frac{1}{L+1} \sum_{\ell=0}^Lp(h_T|\theta_\ell, \pi)}{p(h_T|\pi)} \right]
\end{align}
and we have $\log \frac{\frac{1}{L+1}\sum_{\ell=0}^L p(h_T|\theta_\ell,\pi) }{p(h_T|\pi)} \rightarrow 0$ almost surely as $L\rightarrow \infty$.
The minor technical assumption, which is required to establish convergence is that there exist some $0<\kappa_1,\kappa_2<\infty$ such that\footnote{In practice, we
can actually weaken this assumption significantly if necessary by making $\kappa_1$ and $\kappa_2$ dependent on $h_T$ and $\theta$ then assuming that 
the expectation
$\E_{p(\theta_0, h_T|\pi)} \E_{p(\theta_{1:L})} [\log |\kappa_i(\theta_{j},h_T)|]$ is finite for $i\in\{1,2\}$ and $j\in\{0,1\}$.  This then permits $\kappa_1(h_T,\theta)\to0$ and 
$\kappa_2(h_T,\theta)\to\infty$ for certain $h_T$ and $\theta$, provided that these events are zero measure under both $p(\theta, h_T|\pi)$ and $p(\theta)p(h_T|\pi)$,
thereby avoiding potential issues with tail behavior in the limits of extreme values for $\theta$.}
\begin{equation}
\label{eq:thm2_assumption}
\kappa_1 \le \frac{p(h_T|\theta,\pi)}{p(h_T|\pi)} \le \kappa_2 \quad \forall \theta, h_T.
\end{equation}
using this assumption, the integrand of \eqref{eq:thm2_diff_2} is bounded, because
\begin{align}
	\left|\log \frac{\frac{1}{L+1}\sum_{\ell=0}^L p(h_T|\theta_\ell,\pi) }{p(h_T|\pi)}\right| &= \left| \log \left(\frac{1}{L+1}\sum_{\ell=0}^L \frac{p(h_T|\theta_\ell,\pi) }{p(h_T|\pi)} \right)\right| \\
	&\le \max\left( \left| \log \left(\max_\ell \frac{p(h_T|\theta_\ell,\pi) }{p(h_T|\pi)} \right)\right|,  \left| \log \left(\min_\ell \frac{p(h_T|\theta_\ell,\pi) }{p(h_T|\pi)} \right)\right| \right) \\
	&\le \max\left( \left| \log \kappa_2 \right|,  \left| \log \kappa_1 \right| \right) \\
	& < \infty.
\end{align}
Thus, the Bounded Convergence Theorem can be applied to conclude that $\mathcal{I}_T(\pi) - \mathcal{L}_T(\pi,L) \to 0$ as $L\to\infty$.

Finally,  for the rate of convergence we apply the inequality $\log x \leq x - 1$ to \eqref{eq:thm2_diff} to get
\begin{align}
 \mathcal{I}_T(\pi) - \mathcal{L}_T(\pi,L) &=  \E_{p(\theta_0, h_T|\pi)} \E_{p(\theta_{1:L})} \left[ \log \frac{\frac{1}{L+1} \sum_{\ell=0}^Lp(h_T|\theta_\ell, \pi)}{p(h_T|\pi)} \right] \\
 & \leq  \E_{p(\theta_0, h_T|\pi)} \E_{p(\theta_{1:L})} \left[  \frac{\frac{1}{L+1} \sum_{\ell=0}^Lp(h_T|\theta_\ell, \pi)}{p(h_T|\pi)} - 1 \right] \\
 & =   \E_{p(\theta_0, h_T|\pi)}  \left[  \frac{\frac{1}{L+1} \left( p(h_T|\theta_0 \pi)+ \sum_{\ell=1}^L \E_{p(\theta_{1:L})} [p(h_T|\theta_\ell, \pi)]  \right) }{p(h_T|\pi)} - 1 \right] \\
 & =  \E_{p(\theta_0, h_T|\pi)}  \left[  \frac{\frac{1}{L+1} \left( p(h_T|\theta_0 \pi)+Lp(h_T|\pi) \right) }{p(h_T|\pi)} - 1 \right] \\
 & = \frac{1}{L+1}    \E_{p(\theta_0, h_T|\pi)} \left[  \frac{p(h_T|\theta_0, \pi)}{p(h_T|\pi)} - 1 \right] \\
 & = \frac{C}{L+1}  ,
\end{align}
where we can conclude $C<\infty$ using \eqref{eq:thm2_assumption}.
Combining this with the our previous result showing that $\mathcal{L}_T(\pi,L)$ is a lower bound on $\mathcal{I}_T(\pi)$, we have shown that
\begin{equation}
	0 \le  \mathcal{I}_T(\pi) - \mathcal{L}_T(\pi,L) \le \frac{C}{L+1}.
\end{equation}
This establishes the $\mathcal{O}(L^{-1})$ rate of convergence. 
\end{proof}

\permutation*
\textbf{Technical note:} In this statement, the first expectation is with respect to $p(h_T|\pi)$ for policy $\pi$ and the second is with respect to $p(h_T|\pi')$, where for $t>k$ we set $\pi'(h_t) = \pi(\sigma^{-1}(h_t))$ where $\sigma^{-1}$ acts on the first $k$ labels by permutation and as the identity on other labels. This means we remove \emph{explicit} variability under permutation caused by $\pi$, and show that no other source of variability can arise.
\begin{proof}
	To begin, we set up some notation. Given the partial history $h_k=h_k^1$, we complete the experiment by sampling $(\xi_t,y_t)$ for $t=k+1,...,T$. We denote the resulting full history as $h_T^1$, and define $h_T^2$ similarly.
	Next, we use Theorem~\ref{proposition:terminal} to rewrite the conditional objective under consideration as
	\begin{align}
		\E_{p(h_T^1|\pi)} \left[ \sum_{t=1}^{T} I_{h_{t-1}}(\xi_t) \middle| h_k = h^1_k\right] &= \E_{p(\theta|h_k^1)\prod_{t=k+1}^T p(y_t|\theta,\xi_t)}\left[\log p(h_T^1|\theta,\pi) - \log p(h_T^1|\pi) \right] \\
		&=\E_{p(\theta|h_k^1)\prod_{t=k+1}^T p(y_t|\theta,\xi_t)}\left[\log p(\theta|h_T^1) - \log p(\theta) \right] \\
		&=\E_{p(\theta|h_k^1)p(h_T^1|h_k^1,\theta,\pi)}\left[\log p(\theta|h_T^1) - \log p(\theta) \right].
	\end{align}
	A central point of the proof is that the posterior distribution $p(\theta|h_t)$ is invariant to the order of the history. Indeed, we have
	\begin{equation}
	p(\theta|h_t) \propto p(\theta) \prod_{s=1}^t p(y_s|\theta,\xi_s)
	\end{equation}
	which shows that $p(\theta|h_k^1) = p(\theta|h_k^2)$. Given a continuation of the history $(\xi_{k+1},y_{k+1}),...,(\xi_T,y_T)$, if we use the same continuation starting from $h_k^1$ and $h_k^2$ to give $h_T^1$ and $h_T^2$ then we have $p(\theta|h_T^1) = p(\theta|h_T^2)$. However, we need to show that the continuations $(\xi_{k+1},y_{k+1}),...,(\xi_T,y_T)$ are equal in distribution.
	
	We now show that the sampling distributions of $(\xi_{k+1},y_{k+1}),...,(\xi_T,y_T)$ are equal starting from $h_k^1$ and $h_k^2$. We have shown that $\theta\sim p(\theta|h_k^1)$ is unchanged in distribution if we instead sample $\theta \sim p(\theta|h_k^2)$. Further, we have
	\begin{align}
		\xi_{k+1}^1 = \pi(h_k^1) \qquad \xi_{k+1}^2 = \pi'(h_k^2)
	\end{align}
	which, by the construction of $\pi'$ implies $\xi_{k+1}^1 = \xi_{k+1}^2$. Together, these results imply that the observations $y_{k+1}^1$ and $y_{k+1}^2$ are equal in distribution. Proceeding inductively, since $h_{k+1}^1$ and $h_{k+1}^2$ are equal in distribution a similar argument shows that $h_{k+2}^1$ and $h_{k+2}^2$ have the same distribution. Continuing in this way, we have that $h_T^1$ and $h_T^2$ are equal in distribution. Together, these results imply that
	\begin{equation}
	\label{eq:symmetry_proof}
	\E_{p(\theta|h_k^1)p(h_T^1|h_k^1,\theta,\pi)}\left[\log p(\theta|h_T^1) - \log p(\theta) \right] = \E_{p(\theta|h_k^2)p(h_T^2|h_k^2,\theta,\pi')}\left[\log p(\theta|h_T^2) - \log p(\theta) \right]
	\end{equation}
	which conclude the first part of the proof.
	
	To establish the permutation invariance of the optimal policy $\pi^*$, we reason by induction starting with $k=T-1$, using a dynamic programming style argument. Given $h_{T-1}$, the total EIG is a function of $p(\theta|h_{T-1})$ and $\xi_T$. Since we do not need to account for future asymmetry in the policy, we immediately have that the optimal final design $\xi_T$ only depends on $p(\theta|h_{T-1})$, which implies that is invariant to the order of the history.
	
	We now assume that the optimal policy is permutation invariant starting from $k+2$. Using the previous result \eqref{eq:symmetry_proof}, we separete out the design $\xi_{k+1}$ and substitute $\pi^*$ for both $\pi$ and $\pi'$ (since it is permutation invariant for the steps after $k+1$ by inductive hypothesis) to give
	\begin{equation}
	\begin{split}
	&\E_{p(\theta|h_k^1)p(y_{k+1}|\theta,\xi_{k+1})\prod_{t=k+2}^T p(y_t|\theta,\pi^*(h_{t-1}))}\left[\log p(\theta|h_T^1) - \log p(\theta) \right] \\ &= \E_{p(\theta|h_k^2)p(y_{k+1}|\theta,\xi_{k+1})\prod_{t=k+2}^T p(y_t|\theta,\pi^*(h_{t-1}))}\left[\log p(\theta|h_T^2) - \log p(\theta) \right].
	\end{split}
	\label{eq:bellmanesque}
	\end{equation}
	To extend the optimal policy to $k+1$, we consider choosing $\xi_{k+1}$ and then following $\pi^*$ thereafter.
	As \eqref{eq:bellmanesque} shows us, the decision problem for $\xi_{k+1}$ is the same starting from $h_k^1$ and $h_k^2$ because the posterior distributions $p(\theta|h_k^1)$ and $p(\theta|h_k^2)$ are equal, and the optimal policy after $k+1$ does not depend on history order. This implies that the optimal choice of $\xi_{k+1}$ is the same for $h_k^1$ and $h_k^2$.
	This implies that the optimal policies starting in $h_k^1$ and $h_k^2$ are the same. This completes the proof.	
\end{proof}

\begin{restatable}{theorem}{seqnmc}
	\label{thm:seqnmc}
	For a design function $\pi$ and a number of contrastive samples $L\ge 1$, let
	\begin{equation}
	\mathcal{U}_T(\pi,L) = \E\left[ \log \frac{p(h_T|\theta_0,\pi)}{\frac{1}{L}\sum_{\ell=1}^L p(h_T|\theta_\ell,\pi) } \right]
	\end{equation}
	where the expectation is over $\theta_0,h_T \sim p(\theta,h_T|\pi)$ and $\theta_{1:L}\sim p(\theta)$ independently. Then,
	\begin{equation}
	\mathcal{U}_T(\pi,L) \downarrow \mathcal{I}_T(\pi) \text{ as } L \to \infty
	\end{equation}
	at a rate $\mathcal{O}(L^{-1})$.
\end{restatable}
\begin{proof}

We first show $ \mathcal{U}_T(\pi,L)$ is an upper bound to $\mathcal{I}_T(\pi)$
\begin{align}
 \mathcal{U}_T(\pi,L) - \mathcal{I}_T(\pi)  &= \E_{p(\theta_0, h_T|\pi)} \E_{p(\theta_{1:L})} \left[ \log \frac{p(h_T|\theta_0,\pi)}{\frac{1}{L}\sum_{\ell=1}^L p(h_T|\theta_\ell,\pi) } \right]  -  \E_{p(\theta_0, h_T| \pi)}\left[ \log \frac{p(h_T|\theta_0,\pi)}{p(h_T|\pi)} \right] \\
 &= \E_{p(\theta_0, h_T|\pi)} \E_{p(\theta_{1:L})} \left[ \log {p(h_T|\pi) }-\log\left({ \frac{1}{L} \sum_{\ell=1}^Lp(h_T|\theta_\ell, \pi) } \right) \right] \\
 \intertext{now using Jensen's inequality}
 &\ge \E_{p(\theta_0, h_T|\pi)}  \left[ \log {p(h_T|\pi) }-\log\left( \frac{1}{L} \sum_{\ell=1}^L \E_{p(\theta_{\ell})} \left[p(h_T|\theta_\ell, \pi) \right] \right) \right] \\
 &= \E_{p(\theta_0, h_T|\pi)}  \left[ \log {p(h_T|\pi) }-\log\left( \frac{1}{L} \sum_{\ell=1}^L p(h_T|\pi) \right) \right] \\
 &= \E_{p(\theta_0, h_T|\pi)}  \left[ \log {p(h_T|\pi) }-\log p(h_T|\pi) \right] \\
 &= 0.
\end{align}

To show monotonicity in $L$, pick $L_2\geq L_1 \geq 0$ and consider the difference
\begin{align}
\delta :=  \mathcal{U}_T(\pi,L_1) - \mathcal{U}_T(\pi,L_2) &=  \E_{p(\theta_0, h_T|\pi)} \E_{p(\theta_{1:L_2})} \left[  \log  \frac{  \frac{1}{L_2}\sum_{j=1}^{L_2} p(h_T|\theta_j,\pi) }{\frac{1}{L_1}\sum_{i=1}^{L_1} p(h_T|\theta_i,\pi) }   \right] .
\end{align} 

Notice that we can write expression in the numerator $ \frac{1}{L_2}\sum_{j=1}^{L_2} p(h_T|\theta_j,\pi) = \E_{J_1, \dots, J_{L_1}} \left[ \frac{1}{L_1} \sum_{k=1}^{L_1} p(h_T|\theta_{J_k},\pi)\right]$, where the indices $J_k$ have been uniformly drawn from $1, \dots, L_2$.  
We have
\begin{align}
\delta &=  \E_{p(\theta_0, h_T|\pi)} \E_{p(\theta_{1:L_2})} \left[ \log \E_{J_1, \dots, J_{L_1}}  \left[ \frac{1}{L_1} \sum_{k=1}^{L_1} p(h_T|\theta_{J_k},\pi)\right] - \log \frac{1}{L_1}\sum_{i=1}^{L_1} p(h_T|\theta_i,\pi) \right] \\
\intertext{now applying Jensen's Inequality}
&\ge \E_{p(\theta_0, h_T|\pi)} \E_{p(\theta_{1:L_2})} \left[  \E_{J_1, \dots, J_{L_1}}  \left[\log \frac{1}{L_1} \sum_{k=1}^{L_1} p(h_T|\theta_{J_k},\pi)\right] - \log \frac{1}{L_1}\sum_{i=1}^{L_1} p(h_T|\theta_i,\pi) \right] \\
\intertext{then use the fact that any $L_1$-subset of $\theta_1,...,\theta_{L_2}$ has the same distribution}
&=   \E_{p(\theta_0, h_T|\pi)} \E_{p(\theta_{1:L_2})} \left[  \log \frac{1}{L_1}\sum_{i=1}^{L_1} p(h_T|\theta_i,\pi)    - \log \frac{1}{L_1}\sum_{i=1}^{L_1} p(h_T|\theta_i,\pi)  \right] = 0
\end{align}
which establishes monotonicity.

Finally, convergence is shown analogously to Theorem \ref{thm:seqace}. Again we adopt the assumption~\eqref{eq:thm2_assumption}.  The Strong Law of Large Numbers gives us almost sure convergence $\log \left({\frac{1}{L}\sum_{\ell=1}^L p(h_T|\theta_\ell,\pi) }\right) \rightarrow \log {p(h_T|\pi)}$ as $L\to\infty$. Applying the Bounded Convergence Theorem, as in Theorem~\ref{thm:seqace}, we have
\begin{align}
\lim_{L\rightarrow\infty} \left(\mathcal{U}_T(\pi,L) - \mathcal{I}_T(\pi,L)\right)  &= \E_{p(\theta_0, h_T|\pi)} \E_{p(\theta_{1:L})} \left[\lim_{L\rightarrow\infty} \log  \frac{p(h_T|\pi)}{\frac{1}{L}\sum_{\ell=1}^L p(h_T|\theta_\ell,\pi) } \right] \\
&=0.
\end{align}
Finally, for the rate of convergence, we have
\begin{align}
	\mathcal{U}_T(\pi,L) - \mathcal{I}_T(\pi) &=  \E_{p(\theta_0, h_T|\pi)} \E_{p(\theta_{1:L})} \left[ \log \frac{p(h_T|\pi)}{\frac{1}{L} \sum_{\ell=1}^Lp(h_T|\theta_\ell, \pi)} \right] \\
	&=  \E_{p(\theta_0, h_T|\pi)} \E_{p(\theta_{1:L})} \left[ -\log  \left(\frac{1}{L} \sum_{\ell=1}^L\frac{p(h_T|\theta_\ell, \pi)}{p(h_T|\pi)}\right) \right] \\
	&=  \E_{p(\theta_0, h_T|\pi)} \E_{p(\theta_{1:L})} \left[ -\log  \left(1 + \frac{1}{L} \sum_{\ell=1}^L\left(\frac{p(h_T|\theta_\ell, \pi)}{p(h_T|\pi)}-1\right)\right) \right] \\
	&=  \E_{p(\theta_0, h_T|\pi)} \E_{p(\theta_{1:L})} \left[ \sum_{n=1}^\infty (-1)^n \frac{x^n}{n} \right]
\end{align}
where $x = \frac{1}{L} \sum_{\ell=1}^L\left(\frac{p(h_T|\theta_\ell, \pi)}{p(h_T|\pi)}-1\right)$ and we have applied the Taylor expansion for $\log(1+x)$. We have
\begin{align}
\E_{p(\theta_0, h_T|\pi)} \E_{p(\theta_{1:L})} \left[x \right] &=0\\ 
\E_{p(\theta_0, h_T|\pi)} \E_{p(\theta_{1:L})} \left[ x^2 \right] &= \frac{1}{L}\E_{p(\theta_0, h_T|\pi)} \E_{p(\theta_{1:L})} \left[ \left( \frac{p(h_T|\theta_\ell, \pi)}{p(h_T|\pi)}-1 \right)^2 \right]
\end{align}
and higher order terms are $o(L^{-1})$ \citep{angelova2012moments,nowozin2018debiasing}. This shows that $\mathcal{U}_T(\pi,L) - \mathcal{I}_T(\pi)\to 0$ at a rate $\mathcal{O}(L^{-1})$. 
This concludes the proof.
\end{proof}

\section{Additional bounds}
\label{sec:app:bounds}
In this section, we consider a more general lower bound on $\mathcal{I}_T(\pi)$ based on the ACE bound of~\citet{foster2020unified}.
We consider a parametrized proposal distribution $q(\theta;h_T)$ which can be used to approximate the posterior $p(\theta|h_T)$. One example of such a proposal would be an amortized variational approximation to the posterior that takes as input $h_T$ and outputs a variational distribution over $\theta$. It would be possible to share the representation $R(h_T)$ from \eqref{eq:defnofr} between the design network and the inference network. However, the following theorem is not limited to variational posteriors, and concerns any parametrized proposal distribution.
\begin{restatable}{theorem}{additionalbounds}
	\label{thm:additionalbounds}
	For a design function $\pi$, a number of contrastive samples $L\ge 1$, and a parametrized proposal $q(\theta;h_T)$, we have the sequential Adaptive Contrastive Estimation (sACE) lower bound
	\begin{equation}
	\mathcal{I}_T(\pi) \ge \E_{p(\theta_0,h_T|\pi)q(\theta_{1:L};h_T)}\left[ \log \frac{p(h_T|\theta_0,\pi)}{\frac{1}{L+1}\sum_{\ell=0}^L \frac{p(h_T|\theta_\ell,\pi)p(\theta_\ell) }{q(\theta_\ell;h_T)}} \right]
	\end{equation}
	and the  sequential Variational Nested Monte Carlo (sVNMC) upper bound
	\begin{equation}
		\mathcal{I}_T(\pi) \le  \E_{p(\theta_0,h_T|\pi)q(\theta_{1:L};h_T)}\left[ \log \frac{p(h_T|\theta_0,\pi)}{\frac{1}{L}\sum_{\ell=1}^L \frac{p(h_T|\theta_\ell,\pi)p(\theta_\ell) }{q(\theta_\ell;h_T)} } \right].
	\end{equation}
\end{restatable}
\begin{proof}
	We begin by showing the sACE lower bound. The proof closely follows that of Theorem~\ref{thm:seqace}. We have the error term
	\begin{align}
		\delta_{sACE} &= \E_{p(\theta_0, h_T| \pi)}\left[ \log \frac{p(h_T|\theta_0,\pi)}{p(h_T|\pi)} \right]  - \E_{p(\theta_0, h_T|\pi)} \E_{q(\theta_{1:L};h_T)} \left[ \log \frac{p(h_T|\theta_0,\pi)}{\frac{1}{L+1}\sum_{\ell=0}^L \frac{p(h_T|\theta_\ell,\pi)p(\theta_\ell)}{q(\theta_\ell;h_T)} } \right] \\
		&= \E_{p(\theta_0, h_T|\pi)} \E_{q(\theta_{1:L};h_T)} \left[ \log \frac{\frac{1}{L+1} \sum_{\ell=0}^L \frac{p(h_T|\theta_\ell,\pi)p(\theta_\ell)}{q(\theta_\ell;h_T)} }{p(h_T|\pi)} \right] \label{eq:thm5_diff} \\
		&= \E_{p(\theta_0, h_T|\pi)} \E_{q(\theta_{1:L};h_T)} \left[ \log \left(\frac{1}{L+1}\sum_{\ell=0}^L \frac{p(\theta_\ell|h_T)}{q(\theta_\ell;h_T)} \right)\right] \\
		\intertext{now introducing the shorthand $q(\theta_{0:L}^{-\ell};h_T):=q\left(\theta_{0:L\backslash\{\ell\}};h_T\right)=\prod_{j=0,j\neq \ell}^L q(\theta_j;h_T)$,}
		&= \E_{p(\theta_0, h_T|\pi)} \E_{q(\theta_{1:L};h_T)} \left[ \log \frac{ \frac{1}{L+1}\sum_{\ell=0}^Lp(\theta_\ell|h_T) q(\theta_{0:L}^{-\ell};h_T) }{ q(\theta_{0:L};h_T)} \right]. \\
		\intertext{Now by the systemtry on term in side the log, we see that this expectation would be the same if it were instead taken over $p(\theta_i, h_T|\pi )q(\theta_{0:L}^{-i};h_T)$ for any $i\in\{0,\dots,L\}$.  It is also the same if we take the expectation over $\frac{1}{L+1} \sum_{i=0}^L p(\theta_i, h_T|\pi )q(\theta_{0:L}^{-i};h_T)=p(h_T|\pi)\frac{1}{L+1} \sum_{i=0}^L p(\theta_i|h_T) q(\theta_{0:L}^{-i};h_T)$
			and thus we have}
		&= \E_{p(h_T|\pi)} \E_{\frac{1}{L+1} \sum_{i=0}^L p(\theta_i|h_T) q(\theta_{0:L}^{-i};h_T)} \left[ \log \frac{  \frac{1}{L+1}\sum_{\ell=0}^Lp(\theta_\ell|h_T) q(\theta_{0:L}^{-\ell};h_T)  }{q(\theta_{0:L};h_T) } \right] \\
		&=\E_{p(h_T|\pi)} \big[ \textnormal{KL}\left( \breve{q}(\theta_{0:L} ;h_T)  || q(\theta_{0:L};h_T) \right)  \big]  \label{eq:thm5_KL}
	\end{align}
	where 
	$\breve{q}(\theta_{0:L} ;h_T) =  \frac{1}{L+1}\sum_{\ell=0}^L p(\theta_\ell|h_T) q(\theta_{0:L}^{-\ell};h_T)$, which is indeed a distribution since
	\begin{equation}
	\int \breve{q}(\theta_{0:L} ;h_T) d\theta_{0:L} = \frac{1}{L+1}\sum_{\ell=0}^L\left(  \int p(\theta_\ell|h_T) d\theta_\ell \cdot \int q(\theta_{0:L}^{-\ell};h_T) d\theta_{0:L}^{-\ell} \right) = 1.
	\label{eq:thm5_ptilde_integral}
	\end{equation}
	Now by Gibb's inequality the expected KL in~\eqref{eq:thm5_KL} must be non-negative,  establishing the required lower bound.
	
	Turning to the sVNMC bound, we use a proof that is close in spirit to Theorem~\ref{thm:seqnmc}. We have the error term
	\begin{align}
		\delta_{sVNMC} &= \E_{p(\theta_0, h_T|\pi)} \E_{q(\theta_{1:L};h_T)} \left[ \log \frac{p(h_T|\theta_0,\pi)}{\frac{1}{L}\sum_{\ell=1}^L \frac{p(h_T|\theta_\ell,\pi)p(\theta_\ell)}{q(\theta_\ell;h_T)} } \right]  -  \E_{p(\theta_0, h_T| \pi)}\left[ \log \frac{p(h_T|\theta_0,\pi)}{p(h_T|\pi)} \right] \\
		&= \E_{p(\theta_0, h_T|\pi)} \E_{q(\theta_{1:L};h_T)} \left[ \log {p(h_T|\pi) }-\log\left({ \frac{1}{L} \sum_{\ell=1}^L\frac{p(h_T|\theta_\ell, \pi)p(\theta_\ell)}{q(\theta_\ell;h_T)} } \right) \right] \\
		\intertext{now using Jensen's inequality}
		&\ge \E_{p(\theta_0, h_T|\pi)}  \left[ \log {p(h_T|\pi) }-\log\left( \frac{1}{L} \sum_{\ell=1}^L \E_{q(\theta_{\ell};h_T)} \left[\frac{p(h_T|\theta_\ell, \pi)p(\theta_\ell)}{q(\theta_\ell;h_T)} \right] \right) \right] \\
		&=\E_{p(\theta_0, h_T|\pi)}  \left[ \log {p(h_T|\pi) }-\log\left( \frac{1}{L} \sum_{\ell=1}^L \E_{p(\theta_{\ell})} \left[p(h_T|\theta_\ell, \pi) \right] \right) \right] \\
		&= \E_{p(\theta_0, h_T|\pi)}  \left[ \log {p(h_T|\pi) }-\log\left( \frac{1}{L} \sum_{\ell=1}^L p(h_T|\pi) \right) \right] \\
		&= \E_{p(\theta_0, h_T|\pi)}  \left[ \log {p(h_T|\pi) }-\log p(h_T|\pi) \right] \\
		&= 0.
	\end{align}
This establishes the upper bound.
\end{proof}

%% file: appendix_gradients.tex
\section{Gradient details}
\label{sec:app:method}

\subsection{Score function gradient}\label{sec:app_score_grad}
Recall that our sPCE objective is
\begin{align}
	\mathcal{L}_T(\pi_\phi, L) & =  \E_{p(\theta_{0:L})\likelihoodi{0}}\left[ \glo_L(\theta_{0:L},h_T)  \right] \\
	& =  \E_{\priori{0:L}\likelihoodi{0}} \left[ \log \frac{\likelihoodi{0}}{\frac{1}{L+1} \sum_{\ell=0}^L \likelihoodi{\ell} }  \right] \\
	& =  \E_{\priori{0:L}\likelihoodi{0}} \left[ \log \frac{\likelihoodi{0}}{\sum_{\ell=0}^L \likelihoodi{\ell} }  \right] + \log(L+1)
\end{align}

Differentiating this gives:
\begin{align}
	\dydx{\mathcal{L}_T}{\phi} & = \E_{\priori{0:L}} \left[ \int  \dydx{}{\phi}  \left(\likelihoodi{0}  \log \frac{\likelihoodi{0}}{\sum_{\ell=0}^L \likelihoodi{\ell} } \right) d h_T \right] \\
	& = \E_{\priori{0:L}}  \left[\int  \log \frac{\likelihoodi{0}}{\sum_{\ell=0}^L \likelihoodi{\ell} }\dydx{}{\phi} \likelihoodi{0}  +\likelihoodi{0}     \dydx{}{\phi} \log \frac{\likelihoodi{0}}{\sum_{\ell=0}^L \likelihoodi{\ell} } d h_T \right]  \\
	& = \E_{\priori{0:L}}  \Bigg[ ~~  \int  \likelihoodi{0}  \log \frac{\likelihoodi{0}}{\sum_{\ell=0}^L \likelihoodi{\ell} } \left( \dydx{}{\phi} \log \likelihoodi{0} \right)   d h_T  \label{eq:logtrick} \\
	& \qquad\qquad\  + \int \likelihoodi{0} \left( \dydx{}{\phi} \log \likelihoodi{0} \right) d h_T   - \int \likelihoodi{0} \dydx{}{\phi} \log  \sum_{\ell=0}^L \likelihoodi{\ell} d h_T  \Bigg] \\
	& = \E_{\priori{0:L}}\E_{\likelihoodi{0}} \left[  \log \frac{\likelihoodi{0}}{\sum_{\ell=0}^L \likelihoodi{\ell} } \left( \dydx{}{\phi} \log \likelihoodi{0} \right) - \dydx{}{\phi}  \log  \sum_{\ell=0}^L \likelihoodi{\ell} \right]  \label{eq:rev_logtrick}.
\end{align}
In line  \eqref{eq:logtrick} we used the log-trick $\dydx{}{x} f(x) = f(x) \left(  \dydx{}{x} \log f(x) \right)$ and again in line  \eqref{eq:rev_logtrick} (in the reverse direction), together with the fact $\int \dydx{}{\phi}  \likelihoodi{0} dh_T =\dydx{}{\phi} \int  \likelihoodi{0} dh_T = 0$.

\subsection{Expanded reparametrized gradient}
For completeness, we provided a fully expanded form of the gradient in \eqref{eq:lossgrad}, computed using the chain rule.  In practice, derivatives of this form are calculated automatically in PyTorch \citep{pytorch}.

Initially, we set up some additional notation. Suppose $\xi$ the design is of dimension $D_1$ and $y$ the observation is of dimension $D_2$. Then $u = (\xi,y)$ is of dimension $D_1 + D_2$. For an arbitrary scalar quantity $x$, we have
\begin{equation}
	\pypx{x}{u} = \begin{pmatrix}
		\pypx{x}{\xi^{(1)}} & ... & \pypx{x}{\xi^{(D_1)}} & \pypx{x}{y^{(1)}} & ... & \pypx{x}{y^{(D_2)}}
	\end{pmatrix}
\end{equation}
and
\begin{equation}
\pypx{u}{x} = \begin{pmatrix}
\pypx{\xi^{(1)}}{x} & ... & \pypx{\xi^{(D_1)}}{x} & \sum_{d=1}^{D_1} \pypx{y^{(1)}}{\xi^{(d)}}\pypx{\xi^{(d)}}{x} & ... & \sum_{d=1}^{D_1} \pypx{y^{(D_2)}}{\xi^{(d)}}\pypx{\xi^{(d)}}{x}
\end{pmatrix}^\top.
\end{equation}
This notation enables us to concisely and clearly deal with both scalar and vector quantities. In general, the derivatives $\partial a/\partial b$ and $da/db$ represent a matrix of shape $(\dim a, \dim b)$ where one or both of $a, b$ may have dimension 1. 
This notation is particularly attractive because the Chain Rule for partial derivatives can be concisely expressed as follows. Suppose $a = a(b_1(c),...,b_n(c), c)$, then the total derivative is given by
\begin{equation}
	\dydx{a}{c} = \pypx{a}{c} + \sum_{i=1}^n \pypx{a}{b_i}\dydx{b_i}{c}
\end{equation}
where the normal rules of matrix multiplication apply.
We now apply this in the context of the function $g(\theta_{0:L},h_T)$ which was defined in Section~\ref{sec:gradientestimation}.

We have $g = g(\theta_{0:L},u_1,...,u_T)$. The Chain Rule implies that
\begin{equation}
	\dydx{g}{\phi} = \sum_{t=1}^T \pypx{g}{u_t}\dydx{u_t}{\phi}.
\end{equation}
We also have, for $t=1,...,T$, that $u_t = u(\phi, h_{t-1},\theta_0, \epsilon_t) = u(\phi,u_1,...,u_{t-1},\theta_0,\epsilon_t)$. This represents the dependence of $\xi_t$ on $h_{t-1}$ via $\pi_\phi$, and the further dependence of $y_t$ on $\theta_0$ and $\epsilon_t$. Expanding the derivatives again using the Chain Rule gives
\begin{equation}
\dydx{g}{\phi} = \sum_{t=1}^T \pypx{g}{u_t} \left(\pypx{u_t}{\phi} + \sum_{s=1}^{t-1}\pypx{u_t}{u_s}\dydx{u_s}{\phi} \right).
\end{equation}
Again, we can expand the total derivative to give
\begin{equation}
\label{eq:twoformgrad}
\dydx{g}{\phi} = \sum_{t=1}^T \pypx{g}{u_t} \left(\pypx{u_t}{\phi} + \sum_{s=1}^{t-1}\pypx{u_t}{u_s}\left(\pypx{u_s}{\phi} + \sum_{r=1}^{s-1} \pypx{u_s}{u_r}\dydx{u_r}{\phi} \right) \right).
\end{equation}
Rather than continuing in this manner, we observe that the current expansion \eqref{eq:twoformgrad} can be split up as follows
\begin{equation}
\dydx{g}{\phi} = \sum_{t=1}^T \pypx{g}{u_t} \pypx{u_t}{\phi} + \sum_{1\le s<t\le T}\pypx{g}{u_t}\pypx{u_t}{u_s}\pypx{u_s}{\phi} + \sum_{1\le r < s < t \le T} \pypx{g}{u_t}\pypx{u_t}{u_s} \pypx{u_s}{u_r}\dydx{u_r}{\phi} 
\end{equation}
which shows that we have completely enumerated over all paths of length 1 and 2 through the computational graph, and the final term with a total derivative concerns paths of length 3 or more.
This approach can be naturally extended to enumerate over all paths. To write this concisely, we introduce a new variable $k$ which denotes the length of the path, and then a sum over all increasing sequences $1\le t_1<...<t_k\le T$. This gives
\begin{equation}
\label{eq:maingrad}
\dydx{g}{\phi} = \sum_{k=1}^T \left[ \sum_{1\le t_1<...<t_k\le T} \pypx{g}{u_{t_k}} \pypx{u_{t_k}}{u_{t_{k-1}}} ... \pypx{u_{t_2}}{u_{t_1}} \pypx{u_{t_1}}{\phi}\right].
\end{equation}
This can be written concisely as
\begin{equation}
	\dydx{g}{\phi} = \sum_{\substack{k\in\{1,\dots,T\} \\ 1\le t_1<...<t_k \le T}} \pypx{\glo}{u_{t_k}} \left( \prod_{j=1}^{k-1} \pypx{u_{t_{j+1}}}{u_{t_j}}\right) \pypx{u_{t_1}}{\phi}
\end{equation}
where the product is interpretted in the order given in \eqref{eq:maingrad} for the matrix multiplication to operate correctly, and an empty product is equal to the identity.

%% file: appendix_locfin.tex

\section{Experiment details}
\label{sec:exp-appendix}

\begin{figure}[t]
	\centering
	\includegraphics[width=0.6\textwidth]{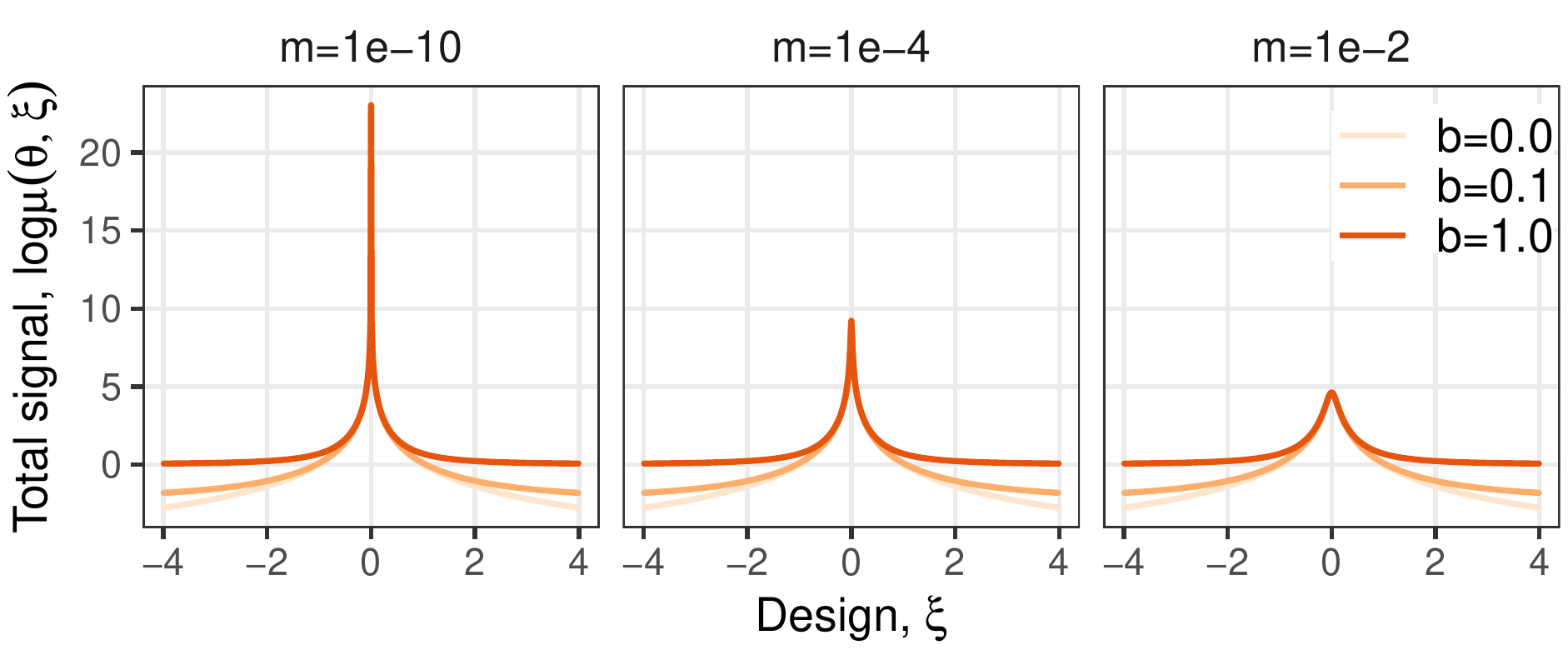}
	\caption{Log-total intensity}
	\label{fig:g_map}
\end{figure}

Our experiments were implemented using PyTorch \citep{pytorch} and Pyro \citep{pyro}. 
An open-source implementation of DAD, including code for reproducing each experiment, is available at \url{https://github.com/ae-foster/dad}. 
Full details on running the code are given in the \texttt{README.md} file. 

\subsection{Location Finding}\label{sec:appendix_locfin}

In this experiment we have $K$ hidden objects or \textit{sources} in $\R^d$, $d \in \{1, 2, 3\}$ and aim to learn their locations, $\theta= \{\theta_k\}_{k=1}^K$. The number of sources, $K$, is assumed to be known. Each of the sources emits a signal with intensity obeying the inverse-square law. In other words, if a source is located at $\theta_k$ and we perform a measurement at a point $\xi$, the signal strength will be proportional to $\frac{1}{\|\theta_k-\xi\|^2}$. 

Since there are multiple sources, we consider the total intensity at location $\xi$, which is a superposition of the individual ones
\begin{equation}
\mu(\theta, \xi) = b + \sum_{k=1}^K \frac{\alpha_k}{m+ \|\theta_k-\xi\|^2},
\end{equation}
where $\alpha_k$ can be known constants or random variables,  $b, m>0$ are constants controlling background and maximum  signal,  respectively.   Figure \ref{fig:g_map} shows the effect $b$ and $m$ have on log total signal strength.

We  place a standard normal prior on each of the location parameters $\theta_k$ and we observe the log total intensity with some Gaussian noise. We therefore have the following prior and likelihood:
\begin{equation}
\theta_k \iid N(0_d, I_d), ~\log y \mid \theta, \xi \sim N(\log \mu(\theta, \xi), \sigma).
\end{equation}

The model hyperparameters used in our experiments can be found in the table below. 
\begin{center}
	\begin{tabular}{lr}
		Parameter & Value\\
		\hline
		Number of sources, $K$ & 2 \\
		Base signal, $b$ & $10^{-1}$ \\
		Max signal, $m$ & $10^{-4}$ \\
		$\alpha_1, \alpha_2$ & 1 \\
		Signal noise, $\sigma$ & $0.5$ \\
	
	\end{tabular}
\end{center}

We trained a DAD network to amortize experimental design for this problem, using the neural architecture outlined in Section~\ref{sec:arch}. Both the encoder and the decoder are simple feed-forward neural networks with a single hidden layer; details in the following table. For the encoder
\begin{center}
	\begin{tabular}{llrr}
		Layer & Description & Dimension & Activation \\
		\hline
		Input & $\xi$, $y$ & 3 & - \\
		H1 & Fully connected & 256 & ReLU \\
		Output & Fully connected  & 16 & - \\
	\end{tabular}
\end{center}
and for the emitter
\begin{center}
	\begin{tabular}{llrr}
		Layer & Description & Dimension & Activation \\
		\hline
		Input & $R(h_t)$ & 16 & - \\
		H1 & Fully connected & 2 & - \\
		Output & $\xi$ & 2  & -
	\end{tabular}
\end{center}

Since the likelihood is reparametrizable, we use \eqref{eq:lossgrad} to calculate approximate gradients.  We optimized the network using Adam~\citep{kingma2014adam}  with exponential learning rate annealing with parameter $\gamma$. Full details are given in the following table.

\begin{center}
\begin{tabular}{lr}
	Parameter & Value \\
	\hline
	Inner samples, $L$ & 2000 \\
	Outer samples & 2000 \\
	Initial learning rate & $5\times 10^{-5}$ \\
	Betas & (0.8, 0.998) \\
	$\gamma$ & 0.98 \\
	Gradient steps & 50000 \\
	Annealing frequency & 1000
\end{tabular}
\end{center}

We used a greater number of inner and outer samples for a more accurate estimate of $\mathcal{I}_T(\pi)$ for evaluation when computing the presented values in Table~\ref{tab:locfin_T30_results} and in our Training Stability ablation, specifically $L=5 \times 10^5$ inner samples, and 256 (variational) or 2048 (other methods) outer samples.

\textbf{Deployment times} Deployment speed tests were performed on a CPU-only machine witht the following specifications:
\begin{center}
	\begin{tabular}{ll}
		Memory & 16 GB 2133 MHz LPDDR3 \\
		Processor & 2.8 GHz Quad-Core Intel Core i7 \\
		Operating System & MacOS BigSur v.11.2.3
	\end{tabular}
\end{center}
We took the mean and $\pm1$ s.e. over 10 realizations.
Deployment times for all methods are given in the following table
\begin{center}
	\centering
	\begin{tabular}{lr}
		Method & Deployment time (s) \\
		\hline
		Random & 0.0026 $\pm$ \phantom{0}0.0001 \\
		Fixed &  0.0018 $\pm$ \phantom{0}0.0001   \\
		DAD &   0.0474 $\pm$ \phantom{0}0.0003  \\
		\hline
		Variational &   8963.2\phantom{000} $\pm$ 42.2\phantom{000}
	\end{tabular}
\end{center}

\begin{figure}[t]
	\centering
	\includegraphics[width=0.6\textwidth]{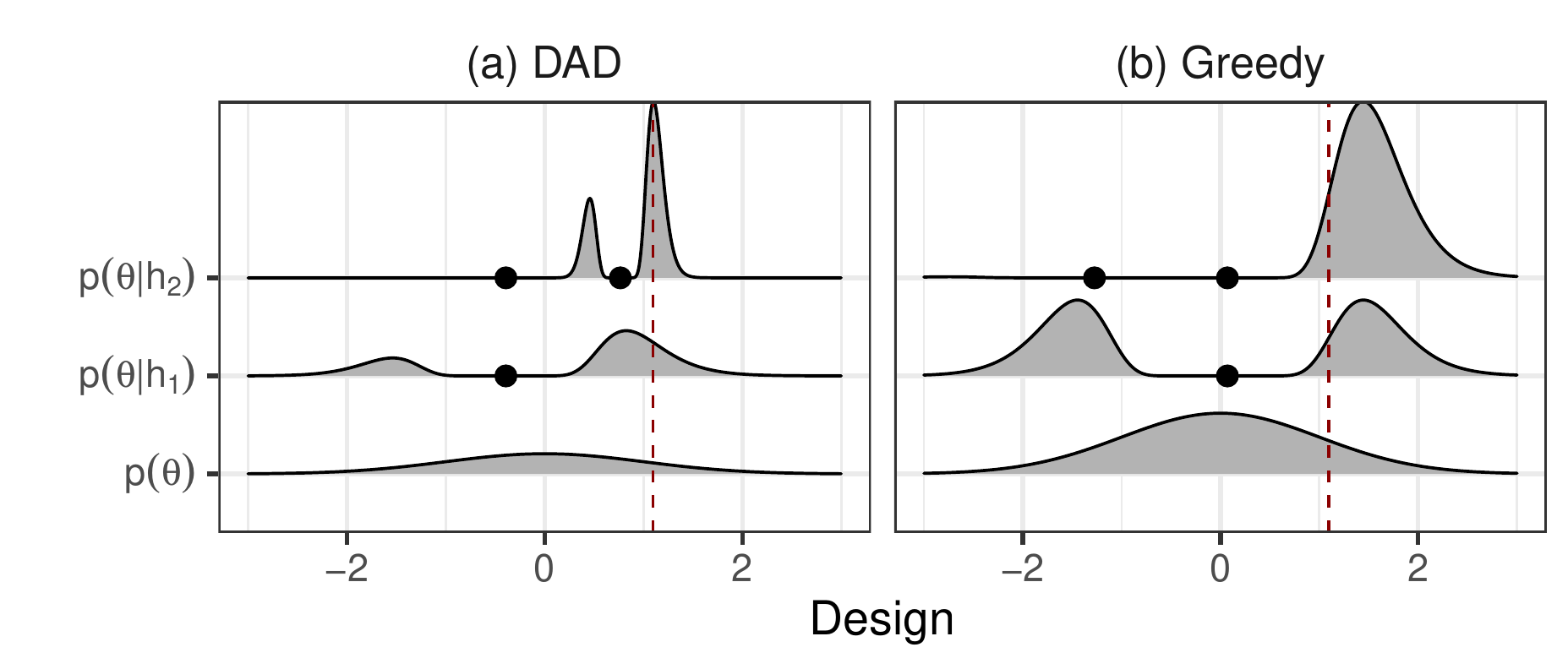}
	\caption{Posterior distributions of the location finding example with $K=1$ source $\R$.}
	\label{fig:locfin_posterior}
\end{figure}

\textbf{Discussion details}
In the discussion, we used a simpler form of the same model with $K=1$ source and $\theta\in\R,\xi\in\R$.
In this simplified setting, we can calculate the true optimal myopic (greedy) baseline  using  numerical integration. We evaluate equation \eqref{eq:mi} using line integrals as follows
\begin{align}
I_t(\xi) =&  \int p(\theta| h_{t-1}) \E_{p(y\vert \theta)} \left[ \log \frac{p(y \vert \theta)}{\int p(\theta' \vert  h_{t-1}) p(y\vert \theta') d\theta' }  \right] d\theta   \\
=&  \int p(\theta| h_{t-1}) \E_{p(y\vert \theta)} \left[ \log \int p(\theta' \vert  h_{t-1}) p(y\vert \theta') d\theta'    \right] d\theta + C \label{eq:greedy_grid}
\end{align}
where $C =- H(p(y\vert\theta))$ is the entropy of a Gaussian, location independent and therefore constant with respect to $\xi$. 
We calculate \eqref{eq:greedy_grid}  for a range of designs,  $\xi \in \Xi_{\text{grid}}$, and select the optimal design $\xi^* = \arg\max_{\Xi_{\text{grid}}} I_t(\xi) $. The integrals themselves are also calculated using numerical integration on a grid, $\Theta_{\text{grid}}$, and use sampling to calculate the inner expectation; further details can in the table below.
\begin{center}
\begin{tabular}{ll}
	Parameter & Value \\
	\hline
	Design grid,   $\Xi_{\text{grid}}$ &  300 equally spaced from -3 to 3  \\ 
	$\theta$ grid,  $\Theta_{\text{grid}}$ & 600 equally spaced from -4 to 4 \\
	$y$ samples for inner expectation & 400 \\
\end{tabular}
\end{center}

It is important to emphasize that even in this simple one-dimensional setting evaluating the myopic strategy is extremely costly and may require more sophisticated numerical integration techniques (e.g. quadrature) as posteriors become more peaked.  Furthermore, as Figure \ref{fig:locfin_posterior} indicates, the resulting posteriors are complex and multi-modal even in 1D.  This multi-modality may also be a reason why the variational method does not work well in this example.

%% file: appendix_death_temporal.tex
\begin{figure}[t!]
	\centering
	\includegraphics[width=0.6\columnwidth]{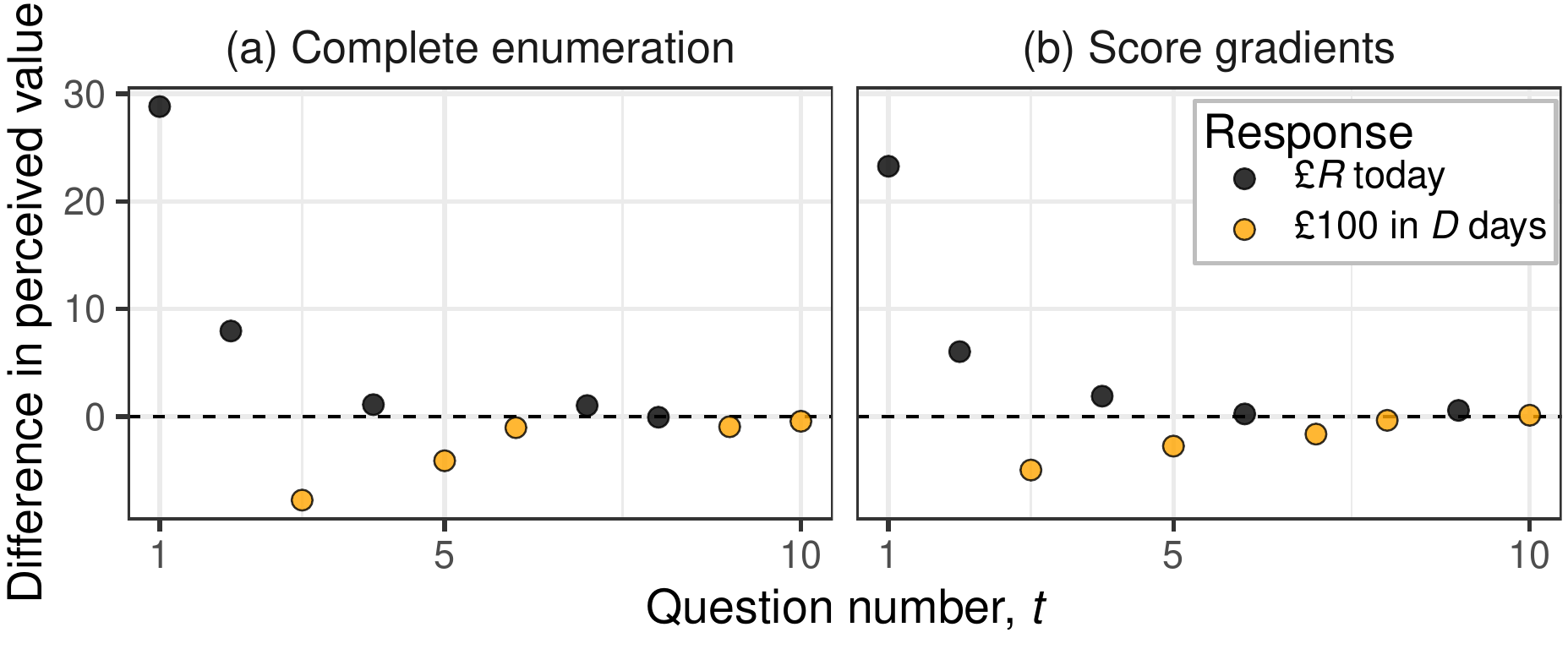}
	\caption{Comparison of two gradient methods for the hyperbolic temporal discounting model with $T=10$ experiments.}
	\label{fig:designs_gradients_comparison}
\end{figure} 

\subsection{Hyperbolic temporal discounting}
\label{sec:app:hyperbolic}

We consider a hyperbolic temporal discounting model \citep{mazur1987adjusting,vincent2016hierarchical,vincent2017} in which a participant's behaviour is characterized by the latent variables $\theta=(k, \alpha)$ with prior distributions
\begin{align}
	\log k \sim N(-4.25, 1.5) \qquad
	\alpha \sim \text{HalfNormal}(0, 2)
\end{align}
where the HalfNormal distribution is a Normal distribution truncated at 0.
For given $k,\alpha$, the value of the two propositions ``\textsterling$R$ today'' and ``\textsterling100 in $D$ days'' with design $\xi=(R,D)$ are given by
\begin{align}
	V_0 = R, \qquad
	V_1 = \frac{100}{1 + kD}.
\end{align}
The probability of the participant selecting the second option, $V_1$, rather than $V_0$ is then modelled as 
\begin{equation}
	p(y=1|k,\alpha,R,D) = \epsilon + (1-2\epsilon) \Phi\left( \frac{V_1-V_0}{\alpha} \right)
\end{equation}
where $\Phi$ is the c.d.f. of the standard Normal distribution, i.e. 
\begin{equation}
	\Phi(z) = \int_{-\infty}^z \frac{1}{\sqrt{2\pi}}\exp\left({-\tfrac{1}{2} z^2 }\right)
\end{equation}
and we fix $\epsilon=0.01$.
We considered the iterated version of this experiment, modelling $T=20$ experiments with each sampled setting for the latents $k,\alpha$.

We began by training a DAD network to amortize experimental design for this problem.
The design parameters $R,D$ have the constraints $D>0$ and $0<R<100$.
We represented $R,D$ in an unconstrained space $\xi_d,\xi_r$ and transformed them using the maps 
\begin{align}
	D = \exp\left( \xi_d  \right) \qquad
	R = 100\ \text{sigmoid}(\xi_r)
\end{align}
We used the neural architecture outlined in Section~\ref{sec:arch}. For the encoder $E_{\phi_1}$ we used the following network with two hidden layers
\begin{center}
	\begin{tabular}{llrr}
		Layer & Description & Dimension & Activation \\
		\hline
		Design input & $\xi_d,\xi_r$ & 2 & - \\
		H1 & Fully connected & 256 & Softplus \\
		H2 & Fully connected & 256 & Softplus \\
		H3 & Fully connected & 16 & - \\
		H3' & Fully connected & 16 & - \\
		Output & $y \odot H3 + (1-y) \odot H3'$ & 16 & -
	\end{tabular}
\end{center}
The emitter network $F_{\phi_2}$ similarly used two hidden layers as follows
\begin{center}
	\begin{tabular}{llrr}
		Layer & Description & Dimension & Activation \\
		\hline
		Input & $R(h_t)$ & 16 & - \\
		H1 & Fully connected & 256 & Softplus \\
		H2 & Fully connected & 256 & Softplus \\
		Output & $\xi_d,\xi_r$ & 2 & - \\
	\end{tabular}
\end{center}

Since the number of experiments we perform is relatively large ($T=20$), we constructed a score function  gradient estimator of \eqref{eq:score_grad} (see also \S~\ref{sec:app_score_grad} for details) and optimized this network with Adam~\citep{kingma2014adam}.We used exponential learning rate annealing with parameter $\gamma$. Full details are given in the following table.
\begin{center}
\begin{tabular}{lr}
	Parameter & Value \\
	\hline
	Inner samples, $L$ & 500 \\
	Outer samples & 500 \\
	Initial learning rate & $10^{-4}$ \\
	Betas & (0.9, 0.999)  \\
	$\gamma$ & 0.96 \\
	Gradient steps & 100000 \\
	Annealing frequency & 1000
\end{tabular}
\end{center}
For the fixed baseline, we used the same optimization settings, except we set the initial learning rate to $10^{-1}$.
We trained the DAD and fixed methods on a machine with 8 Intel(R) Xeon(R) CPU E5-2637 v4 @ 3.50GHz CPUs, one GeForce GTX 1080 Ti GPU, 126 GiB memory running Fedora 32. Note this is \emph{not} the machine used to conudct speed tests.
For the Badapted baseline of \citet{vincent2017}, we used the public code provided at \url{https://github.com/drbenvincent/badapted}. We used 50 PMC steps with 100 particles. For the baselines of \citet{frye2016measuring} and \citet{kirby2009one}, we used the public code provided at \url{https://github.com/drbenvincent/darc-experiments-matlab/tree/master/darc-experiments}, which we reimplemented in Python.
These methods do not involve a pre-training step, except that we did not include time to compute the first design $\xi_1$ within the speed test, as this can be computed before the start of the experiment.

To implement the deployment speed tests fairly, we ran each method on a lightweight CPU-only machine, which more closely mimics the computer architecture that we might expect to deploy methods such as DAD on. The specifications of the machine we used are described below
\begin{center}
	\begin{tabular}{ll}
		Memory & 7.7GiB \\
		Processor & Intel® Core™ M-5Y10c CPU @ 0.80GHz $\times$ 4 \\
		Operating System & Ubuntu 16.04 LTS
	\end{tabular}
\end{center}
The values in Table~\ref{tab:hyperbolic_time} show the mean and standard error of the times observed from 10 independent runs on a idle system.
To make the final evaluation for each method in Table~\ref{tab:hyperbolic_eig}, we computed the sPCE and sNMC bounds using $L=5000$ inner samples and 10000 outer samples of the outer expectation. We present the mean and standard error from the outer expectation over 10000 rollouts.
 
\subsubsection{Ablation: total enumeration}\label{sec:app_temporal_ablation}
We compare the two methods for estimating gradients for the case of discrete observations: total enumeration of  histories (Equation~\ref{eq:total_enum}) and score function gradient estimator  (Equation~\ref{eq:score_grad}). To this end we train DAD networks to perform $T=10$ experiments, which gives rise to a total of $2^{10}=1024$ possible histories. 

Find that the two methods perform the same, both quantitatively and qualitatively. Table~\ref{tbl:temporal_grad_comparison} reports the estimated upper and lower bounds on the mutual information objective, indicating statistically equal performance of the two methods (mean estimates are within 2 standard errors of each other).  Figure~\ref{fig:designs_gradients_comparison} demonstrates the qualitative similarity in the designs learnt by the two networks.

\begin{table}[h]
\begin{center}
\begin{tabular}{lrr} 
          & Lower bound, $\mathcal{L}_{10} $ & Upper bound, $\mathcal{U}_{10}$ \\
          \hline
Complete enumeration  & $4.068  \pm 0.0124$ & $4.090 \pm 0.0126$ \\
Score function gradient & $4.037 \pm 0.0126$ & $4.058 \pm0.0128$ \\
\end{tabular}
\caption {Final lower and upper bounds on the total information $\mathcal{I}_{10}(\pi)$ for the Hyperbolic Temporal Discounting experiment with $T=10$ experiments and different gradient estimation schemes (see \S~\ref{sec:gradientestimation} and \S~\ref{sec:app_score_grad} for details). The bounds are finite sample estimates of $\mathcal{L}_{10}(\pi,L)$ and $\mathcal{U}_{10}(\pi,L)$ with $L=5000$. The errors indicate $\pm1$ s.e. over the sampled histories.}
\label{tbl:temporal_grad_comparison}
\end{center}
\end{table}

\subsection{Death process}

\begin{figure}[t!]
	\centering
	\includegraphics[width=0.3\columnwidth]{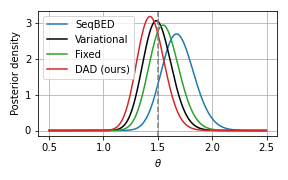}
	\caption{Comparison of posteriors obtained from a single rollout of the Death Process, used to compute the information gains quoted in Section~\ref{sec:death}. The dashed line indicates the true value $\theta=1.5$ used to simulate responses.}
	\label{fig:death_posterior}
\end{figure}

For the Death Process model \citep{cook2008optimal}, we use the settings that were described by \citet{kleinegesse2020sequential}. Specifically, we use a truncated Normal prior for the infection rate
\begin{equation}
\theta \sim \text{TruncatedNormal}(\mu=1, \sigma=1, \text{min}=0, \text{max}=\infty).
\end{equation}
The likelihood is then given by
\begin{equation}
	\eta = 1 - \exp(-\xi  \theta) \qquad y|\theta,\xi \sim \text{Binomial}(N, \eta)
\end{equation}
where we set $N=50$. We consider a sequential version of this experiment as in \citet{kleinegesse2020sequential}, with $T=4$ and in which an independent stochastic process is observed at each step, meaning there are no constraints relating $\xi_1,...,\xi_4$ other than the natural constraint $\xi_t>0$.

We began by training a DAD network to perform experimental design for this problem. 
We used the neural architecture outlined in Section~\ref{sec:arch}. For the encoder $E_{\phi_1}$ we used the following network with two hidden layers
\begin{center}
	\begin{tabular}{llrr}
		Layer & Description & Dimension & Activation \\
		\hline
		Input & $\xi,y$ & 2 & - \\
		H1 & Fully connected & 128 & Softplus \\
		H2 & Fully connected & 128 & Softplus \\
		Output & Fully connected & 16 & - \\
	\end{tabular}
\end{center}
The emitter network $F_{\phi_2}$ similarly used two hidden layers as follows
\begin{center}
	\begin{tabular}{llrr}
		Layer & Description & Dimension & Activation \\
		\hline
		Input & $R(h_t)$ & 16 & - \\
		H1 & Fully connected & 128 & Softplus \\
		H2 & Fully connected & 128 & Softplus \\
		Output & $\xi$ & 1 & Softplus \\
	\end{tabular}
\end{center}

Although the number of experiments we perform is relatively small ($T=4$),  we could not use complete enumeration due to the prohibitively large size of the outcome space ($|\mathcal{Y}|=51$). Hence, we constructed a score function  gradient estimator of \eqref{eq:score_grad} (see also \S~\ref{sec:app_score_grad} for details) and optimized the DAD network with Adam~\citep{kingma2014adam}.We used exponential learning rate annealing with parameter $\gamma$. Full details are given in the following table.\begin{center}
	\begin{tabular}{lr}
		Parameter & Value \\
		\hline
		Inner samples, $L$ & 500 \\
		Outer samples & 500 \\
		Initial learning rate & $0.001$ \\
		Betas & $(0.9, 0.999)$ \\
		$\gamma$ & 0.96 \\
		Gradient steps & 100000 \\
		Annealing frequency & 1000
	\end{tabular}
\end{center}
For the fixed baseline, we used the same optimization settings, except we set the initial learning rate to $10^{-1}$ and we set $\gamma=0.85$.
We trained the DAD and fixed methods using the same machine as used for training in Section~\ref{sec:app:hyperbolic}.
For the variational baseline, we used a truncated Normal variational family to approximate the posterior at each step.
We used SGD with momentum to optimize the design at each step, and to optimize the variational approximation to the posterior at each step. We used exponential learning rate annealing with paramter $\gamma$. The settings used were	
\begin{center}
	\begin{tabular}{lr}
		Parameter & Value \\
		\hline
		Design inner samples & 250 \\
		Design outer samples & 250 \\
		Design initial learning rate & $10^{-2}$ \\
		Design $\gamma$ & 0.9 \\
		Design gradient steps & 5000 \\
		Inference initial learning rate & $10^{-3}$ \\
		Inference $\gamma$ & 0.2 \\
		Inference gradient steps & 5000 \\
		Momentum & 0.1 \\
		Annealing frequency & 1000
	\end{tabular}
\end{center}
For the SeqBED baseline, we used the code publicly available at \url{https://github.com/stevenkleinegesse/seqbed}.
The speed tests except for SeqBED were implemented as in Section~\ref{sec:app:hyperbolic}. For SeqBED and the variational method, we did not include the time to compute the first design as deployment time, as this can be computed before the start of the experiment.
Due to its long-running nature, we implemented the speed test for SeqBED using a more powerful machine with 40 Intel(R) Xeon(R) CPU E5-2680 v2 @ 2.80GHz processors and 189GiB memory. Therefore, the timing value for SeqBED given in Table~\ref{tab:death_eig} represents a significant \emph{under-estimate} of the expected computational time required to deploy this method. However, we note that SeqBED can be applied to a broader class of implicit likelihood models.

For evaluation of $\mathcal{I}_4(\pi)$ in the Death Process, it is possible to compute the information gain $H[p(\theta)] - H[p(\theta|h_T)]$ to high accuracy using numerical integration.
We then took the expectation of the information gain over rollouts, see Table~\ref{tab:death_eig} for the exact number of rollouts used. 
This gives us an estimate
\begin{equation}
	\mathcal{I}_4(\pi) = \E_{p(h_T|\pi)} \left[ H[p(\theta)] - H[p(\theta|h_T)]  \right]
\end{equation}
which is shown to be a valid form for the total EIG in Section~\ref{sec:proofs}.

For a comparison with SeqBED which is too slow to use this evaluation, we instead performed one rollout of each of our methods using a fixed value $\theta=1.5$. This is close in spririt to the evaluation used in \citet{kleinegesse2020sequential}. 
Figure~\ref{fig:death_posterior} shows the posterior distributions obtained from this rollout. The information gains were then computed using the aforementioned numerical integration and are quoted in Section~\ref{sec:death}. We observe that, visually, the posterior distributions are similar, and cluster near to the true value of $\theta$.